\newif\ifFullVersion
\FullVersiontrue 

\documentclass[10pt,final, twocolumn]{IEEEtran}
\IEEEoverridecommandlockouts

\usepackage[all=normal,paragraphs=tight,floats=normal,mathspacing=normal,wordspacing=tight,charwidths=tight,mathdisplays=normal,leading=normal]{savetrees}

\usepackage[final]{graphicx}
\usepackage[dvipsnames]{xcolor}
\usepackage[bookmarks,colorlinks]{hyperref}
\usepackage{acronym, cite} 

\usepackage{float} 
\usepackage{amsmath}
\usepackage{amsthm}
\usepackage{tabularx,bbm}
\usepackage[ruled,linesnumbered,vlined]{algorithm2e}
\usepackage{enumitem}
\usepackage{algorithmic} 
\usepackage{amssymb}
\usepackage{amsmath}
\usepackage{adjustbox}
\usepackage{multirow}
\usepackage{pdfpages}
\usepackage[inkscapelatex=false]{svg}
\usepackage{ctable} 
\usepackage{xcolor} 
\usepackage{colortbl} 
\usepackage{pifont}
\usepackage{bibentry}
\usepackage{diagbox}
\usepackage{bm}

\SetKwInput{KwData}{\textbf{Init}} 

\definecolor{NewColor}{rgb}{0,0,0}
\definecolor{NewColor2}{rgb}{0,0,0}

\newcommand{\myVec}[1]{{\boldsymbol{#1}}}

\newcommand{\mySet}[1]{\mathcal{#1}}
\newcommand{\abs}[1]{{\left| #1 \right|}}

\newtheorem{theorem}{Theorem}
\newtheorem{definition}{Definition}

\newtheorem{proposition}{Proposition}

\newtheorem{lemma}{Lemma}

\newcommand{\cD}{\mathcal{D}}

\newcommand{\cM}{\mathcal{M}}
\newcommand{\cC}{\mathcal{C}}
\newcommand{\cP}{\mathcal{P}}

\newcommand{\cH}{\mathcal{H}}
\newcommand{\cG}{\mathcal{G}}
\newcommand{\cS}{\mathcal{S}}

\newcommand{\cJ}{\mathcal{J}}

\newcommand{\bR}{\mathbb{R}}

\newcommand{\bN}{\mathbb{N}}

\newcommand{\bK}{\mathbb{K}}
\newcommand{\bV}{\mathbb{V}}
\newcommand{\bP}{\mathbb{P}}

\newcommand{\1}{\mathbf{1}}
\newcommand{\E}{\mathbb{E}}

\newcommand{\vx}{\boldsymbol{x}}

\newcommand{\vh}{\boldsymbol{h}}

\newcommand{\vtheta}{\boldsymbol{\theta}}
\newcommand{\epsb}{\bar{\epsilon}}

\DeclareMathOperator*{\sign}{\text{sign}}
\DeclareMathOperator*{\argmin}{arg\,min}
\DeclareMathOperator*{\argmax}{arg\,max}
\DeclareMathOperator*{\lap}{\text{Laplace}}

\acrodef{ml}[ML]{machine learning} 
\acrodef{fl}[FL]{federated learning}
\acrodef{ldp}[LDP]{local differential privacy}
\acrodef{ucb}[UCB]{upper confidence bound}
\acrodef{dp}[DP]{differential privacy}
\acrodef{sgd}[SGD]{stochastic gradient descent}
\acrodef{fa}[FedAvg]{federated averaging} 
\acrodef{ppn}[PPN]{privacy preserving noise} 
\acrodef{iid}[i.i.d.]{independent and identically distributed} 
\acrodef{lm}[LM]{Laplace mechanism}
\acrodef{rv}[RV]{random variable}
\acrodef{mlp}[MLP]{multi-layer perceptron}
\acrodef{cnn}[CNN]{convolutional neural network}
\acrodef{snr}[SNR]{signal-to-noise ratio}
\acrodef{mse}[MSE]{mean squared error}
\acrodef{pbsfl}[PBSFL]{private bandit scheduling for federated learning}
\acrodef{cmab}[CMAB]{combinatorical \ac{mab}}
\acrodef{mab}[MAB]{multi-armed bandit}
\acrodef{dnn}[DNN]{deep neural network}
\acrodef{pause}[PAUSE]{\underline{p}rivacy-aware \underline{a}ctive \underline{u}ser \underline{se}lection}
\acrodef{fc}[FC]{fully-connected network}
\acrodef{cnn}[CNN]{convolutional neural network}
\acrodef{sa}[SA]{simulated annealing}
\acrodef{ucb}[ucb]{upper confidence bound}
\acrodef{cv}[CV]{coefficient of variation}

\graphicspath{{./figs} {./figs_R1} {/figs_journal}} 
\title{
PAUSE: Low-Latency and Privacy-Aware Active User Selection for Federated Learning
}

\author{
\IEEEauthorblockN{Ori Peleg, Natalie Lang, Dan Ben Ami, Stefano Rini, Nir Shlezinger, and Kobi Cohen}

\thanks{Parts of this work were accepted for presentation in the 2025 IEEE International Conference on Acoustics, Speech, and Signal Processing (ICASSP) as the paper \cite{peleg2024pause}.
O. Peleg, N. Lang, D. Ben Ami, N. Shlezinger,  and K. Cohen are with School of ECE, Ben-Gurion University of the Negev, Beer-Sheva, Israel (email: \{oripele, langn, danbenam\}@post.bgu.ac.il;  \{nirshl; yakovsec\}@bgu.ac.il). 
S. Rini is with  the Department of ECE, National Yang-Ming Chiao-Tung University (NYCU), Hsinchu, Taiwan (email: stefano.rini@nycu.edu.tw). 
This research was supported by the Israeli Ministry of Science and Technology.}
}

\begin{document}
\include{enumitem}
\maketitle
%
%
\begin{abstract} 
\Ac{fl} enables multiple edge devices to collaboratively train a machine learning model without the need to share potentially private data.
Federated learning proceeds through iterative exchanges of model updates, which pose two key challenges: (i) the accumulation of privacy leakage over time and (ii) communication latency. These two limitations are typically addressed separately— (i) via perturbed updates to enhance privacy and (ii) user selection to mitigate latency—both at the expense of accuracy.  
%
In this work, we propose a method that jointly addresses the accumulation of privacy leakage and communication latency via active user selection, aiming to improve the trade-off among privacy, latency, and model performance.
To achieve this, we construct a reward function that accounts for these three objectives. Building on this reward, we propose a \ac{mab}-based algorithm,
termed \ac{pause} -- which dynamically selects a subset of users each round while ensuring bounded overall privacy leakage.
We establish a theoretical analysis, systematically showing that the regret growth rate of \ac{pause} follows that of the best-known rate in \ac{mab} literature. 
To address the complexity overhead of active user selection, we propose a simulated annealing-based relaxation of \ac{pause} and analyze its ability to approximate the reward-maximizing policy under reduced complexity.
We numerically validate the privacy leakage, associated improved latency, and accuracy gains of our methods for the federated training in various scenarios.
\end{abstract}

\acresetall

\section{Introduction}
\label{sec:intro} 
 The effectiveness of deep learning models heavily depends on the availability of large amounts of data.
In real-world scenarios, data is often gathered by edge devices such as mobile phones, medical devices, sensors, and vehicles. 
 Because these data often contain sensitive information, there is a pressing need to utilize them for training \acp{dnn} without compromising user privacy.
 A popular framework to enable training \acp{dnn} without requiring data centralization is that of \ac{fl}~\cite{mcmahan2017communication}.  
In \ac{fl}, each participating device locally trains its model in parallel, and a central server periodically aggregates these local models into a global one~\cite{kairouz2021advances}.

The distributed operation of \ac{fl}, and particularly the fact that learning is carried out using multiple remote users in parallel, induces several challenges that are not present in traditional centralized learning~\cite{gafni2022federated,li2020federated}. 
A key challenge stems from the fact that \ac{fl} involves repeated exchanges of highly-parameterized models between the orchestrating server and numerous users. 
This often entails significant communication latency which-- in turn-- impacts convergence, complexity, and scalability~\cite{chen2021communication}. 
Communication latency can be tackled by model compression~\cite{alistarh2018convergence,lang2025olala, han2020adaptive,reisizadeh2020fedpaq,shlezinger2020uveqfed}, and  via over-the-air aggregation in settings where the users share a common wireless channel~\cite{amiri2020machine,sery2020analog, yang2020federated}. 

A complementary approach for balancing communication latency, which is key for scaling \ac{fl} over massive networks, is {\em user selection}~\cite{mayhoub2024review,li2024comprehensive,10197174}. 
User selection limits the number of users participating in each round, traditionally employing pre-defined policies~\cite{xu2020client,abdulrahman2020fedmccs,rizk2022federated,fraboni2021clustered}. 
\textcolor{NewColor}{Alternatively, the user selection can be adapted in an active manner, with a leading framework for active user selection being that of \ac{mab}~\cite{xia2020multi,xu2021online, ami2023client,chen2024personalized, huang2020efficiency, yang2021federated, huang2022stochastic, shi2022vfedcs, wang2024fedmaba, guo2024auction, ZHU2024110512}. \ac{mab} enables active user selection by formulating a dedicated reward, with existing studies formulating reward based on latency~\cite{huang2020efficiency, xu2021online, xia2020multi, ami2023client, chen2024personalized}, class imbalance~\cite{yang2021federated}, unstable clients~\cite{huang2022stochastic, shi2022vfedcs}, and learning progress~\cite{wang2024fedmaba, guo2024auction, ZHU2024110512}.}

Another prominent challenge of \ac{fl} is associated with one of its core motivators--privacy preservation. While \ac{fl} does not involve data sharing, it does not necessarily preserve data privacy, as model inversion attacks were shown to unveil private information and even reconstruct the data from model updates~\cite{zhu2020deep,zhao2020idlg,huang2021evaluating,yin2021see}. The common framework for analyzing privacy leakage in \ac{fl} is based on \ac{ldp}~\cite{kim2021federated}. \ac{ldp} mechanisms  limit privacy leakage in a given \ac{fl} round, typically by employing \ac{ppn}~\cite{wei2020federated,lyu2021dp, lowy2021private}, that can also  be unified with model compression~\cite{lang2022joint,lang2023compressed}. However, this results in having the amount of leaked privacy grow with the number of learning rounds~\cite{dwork2010boosting}, degrading performance by restricting the number of learning rounds and necessitating dominant \ac{ppn}. Existing approaches to avoid accumulation of privacy leakage consider it as a separate task to tackling latency and scalability, often by focusing on a fixed pre-defined number of rounds~\cite{zhang2024dynamic},  or  by relying on an additional trusted coordinator unit~\cite{sun2020ldp,cheu2019distributed, balle2019privacy}, thus deviating from how \ac{fl} typically operates. 
\textcolor{NewColor}{The exploration of unified active user selection policies as means to jointly tackle privacy accumulation and latency in a manner which does not alter the operation of \ac{fl}, i.e., does not require additional infrastructure and/or messages beyond conventional \ac{fl} protocols, was not considered to date, and is the focus of our work.} 

\textcolor{NewColor}{In particular}, we propose a novel framework for private and scalable multi-round \ac{fl} with low latency via {\em active user selection}. Our proposed method, coined {\em \ac{pause}}, is based on a generic per-round privacy budget, designed to avoid leakage surpassing a pre-defined limit for any number of \ac{fl} rounds. This operation results in users inducing more \ac{ppn} each time they participate. The budget is accounted for in formulating a dedicated reward function for active user selection that balances privacy, communication, and generalization. Based on the reward, we propose a \ac{mab}-based policy that prioritizes users with lesser \ac{ppn}, balanced with grouping users of similar expected communication latency and exploring new users for enhancing generalization. We provide an analysis of \ac{pause}, rigorously proving that its regret growth rate obeys the desirable growth in \ac{mab} theory~\cite{zhao2022multi,auer2002finite,chen2013combinatorial}. 


The direct application of \ac{pause} involves a brute search of a combinatorial nature, whose complexity grows dramatically with the number of users.
\textcolor{NewColor}{Nonetheless, we showcase that under some structured dependencies of the reward on the generalization and privacy terms, particularly focusing on settings where these dependencies are given by averaged terms, one can apply \ac{pause} with affordable complexity. We demonstrate this through an efficient algorithm that is shown to implement the desired active selection policy. For the case of generic generalization and privacy dependencies, we circumvent this excessive complexity and enhance scalability by proposing a reduced complexity implementation of \ac{pause} based on \ac{sa}~\cite{hajek1988cooling}, coined {\em SA-\ac{pause}}.} We analyze the computational complexity of SA-\ac{pause}, quantifying its reduction compared to direct \ac{pause}, and rigorously characterize conditions for it to achieve the same performance as costly brute search. 
We evaluate \ac{pause} in learning of different scenarios with varying  \acp{dnn}, datasets, privacy budgets, and data distributions. Our experimental studies systematically show that by fusing privacy enhancement and user selection, \ac{pause} enables accurate and rapid learning, approaching the performance of \ac{fl} without such constraints and notably outperforming alternative approaches that do not account for leakage accumulation. We also show that SA-\ac{pause} approaches the performance of direct \ac{pause} in both privacy leakage, model accuracy, and latency, while supporting scalable implementations on large \ac{fl} networks.

The rest of this paper is organized as follows.
We review some necessary preliminaries and formulate the problem in Section~\ref{sec:System Model}. \ac{pause} is introduced and analyzed in  Section~\ref{sec:Selection}, while its reduced complexity, SA-\ac{pause}, is detailed in Section~\ref{sec:ALSA}. 
Numerical simulations are reported in Section~\ref{sec:Experimental Study}, and Section~\ref{sec:conclusions} provides concluding remarks.

\noindent
\emph{Notation:} Throughout this paper, we use boldface lower-case letters for vectors, e.g., $\vx$. 
The stochastic expectation, probability operator, indicator function, and $\ell_2$ norm are denoted by $\E[\cdot]$, $\bP (\cdot)$, $\1 (\cdot)$,  and $\|\cdot\|$, respectively. For a set $\mathcal{X}$, we write $\abs{\mathcal{X}}$ as its cardinality.

\section{System Model and Preliminaries}
\label{sec:System Model}
This section reviews the necessary background for deriving \ac{pause}.  We start by recalling the \ac{fl} setup and basics in \ac{ldp} in Subsections~\ref{ssec:FL}-\ref{ssec:Preliminaries-LDP}, respectively. Then, we formulate the active user selection problem in Subsection~\ref{ssec:Problem formulation}.

\subsection{Preliminaries: Federated Learning}
\label{ssec:FL} 
\subsubsection{Objective}

The \ac{fl} setup involves the collaborative training of a machine learning model 
\(\vtheta \in \bR^d\), carried out by \(K\) remote users and orchestrated by a server.
Let the set of users be indexed by \(\bK = \{1, \ldots, K\}\), and let \(\cD_k\) denote
the private dataset of user \(k \in \bK\), which cannot be shared with the server.
Define \(F_k(\vtheta)\) as the empirical risk of a model \(\vtheta\) evaluated on \(\cD_k\).
The goal is to determine the \(d \times 1\) optimal parameter vector \(\vtheta^{\rm opt}\)
that minimizes the overall loss across all users, that is
\begin{equation}\label{eq:theta_opt_def}
    \vtheta^{\rm opt} = \argmin_{\vtheta} \left\{F(\vtheta)\triangleq \sum_{k=1}^K \frac{\abs{\cD_k}}{\abs{\cD}} F_k\left(\vtheta\right)\right\}.
\end{equation}

\subsubsection{Learning Procedure}
\ac{fl} operates over multiple iterations divided into rounds~\cite{gafni2022federated}. At \ac{fl} round $t$, the server selects a set of participating users $\mathcal{S}_t \subseteq \bK$, and sends the current model $\vtheta_t$ to them. Each participating user of index $k\in \mathcal{S}_t$ then trains $\vtheta_t$ on its local data $\cD_k$ using, e.g., multiple iterations of mini-batch \ac{sgd}~\cite{li2019convergence}, into the updated  $\vtheta^k_{t+1}$. 

The model update obtained by the $k$th user, denoted $\vh^k_{t+1} = \vtheta^k_{t+1} - \vtheta_t$, is shared with the server, which aggregates the local updates into a global model update. 
The aggregation rule commonly employed by the central server in \ac{fl} is that of \ac{fa}~\cite{mcmahan2017communication}, in which the global model is obtained as
\begin{align}\label{eq:fl_update}
    \vtheta_{t+1} = \vtheta_t + \sum_{k\in \mathcal{S}_t}  \alpha_t^k\vh^k_{t+1}=\sum_{k\in \mathcal{S}_t}  \alpha_t^k  \vtheta^k_{t+1},
\end{align}
where $\alpha_t^k = \frac{\abs{\cD_k}}{\abs{\cup_{j \in \mathcal{S}_t}\cD_j}}$.
The updated global model is again distributed to the users and the learning procedure continues.

\subsubsection{Communication Model}
Communication between the users and the server is associated with some varying latency~\cite{gafni2022federated}. We model this delay via the \acl{rv} $\tau_{t,k}$, representing the  
total latency in the $t$th  round between the server and the $k$th user. Accordingly, the communication latency of the whole round, denoted as $\tau_t^{\rm total}$, is determined by the user with the highest latency
\begin{equation}
\label{eqn: Commlatency}
    \tau_t^{\rm total} = \max_{k \in \mathcal{S}_t} \tau_{t,k}.
\end{equation}

The communication latency $\tau_{t,k}$ varies over time (due to  fading~\cite{chen2021communication}) and between users  (due to system heterogeneity~\cite{lang2024stragglers}). As the latter is device specific, we model $\tau_{t,k}$ as being drawn in an i.i.d. manner from a device specific distribution~\cite{xia2020multi}, denoted $\tau_k$. We further assume the users differ in their expected latencies, $\E[\tau_k]$. We denote the minimal difference between these terms as $\delta \triangleq \min_{i\neq j \in \bK}  |\E[\tau_i] - \E[\tau_j]|$, and assume that there is a minimal latency corresponding to, e.g., the minimal delay. Mathematically, this implies that there exists some  $ \tau_{\min}>0$ such that $\tau_{t,k} \geq \tau_{\min}$ with probability one.

\subsection{Preliminaries: Local Differential Privacy}
\label{ssec:Preliminaries-LDP}
One of the main motivations for \ac{fl} is the need to preserve the privacy of the users' data. Nonetheless, the concealment of the dataset of the $k$th user, $\mathcal{D}_k$, in favor of sharing the model updates trained using $\mathcal{D}_k$, was shown to be potentially leaky~\cite{zhu2020deep,zhao2020idlg,huang2021evaluating,yin2021see}. Therefore, to satisfy the privacy requirements of \ac{fl}, dedicated privacy mechanisms are necessary. 

In \ac{fl}, privacy is commonly quantified in terms of \ac{ldp}~\cite{kasiviswanathan2011can, wang2020federated}, as this metric assumes an untrusted server by the users.

\begin{definition}[$\epsilon$-\ac{ldp}~\cite{wang2020comprehensive}]\label{def:LDP}
A randomized  mechanism $\cM$ satisfies $\epsilon$-\ac{ldp} if for any pairs of input values $v,v'$ in the domain of $\cM$ and for any possible output $y$ in it, it holds that
\begin{align}\label{eq:LDP}
    \bP [\cM(v)=y] \leq e^\epsilon \bP [\cM(v')=y].
\end{align}
\end{definition}
In Definition~\ref{def:LDP}, a smaller $\epsilon$ means stronger privacy protection.
A common mechanism to achieve $\epsilon$-\ac{ldp} is the \ac{lm}. Let
$\lap(\mu,b)$ be the Laplace distribution with location $\mu$ and scale $b$. The \ac{lm} is defined as:
\begin{theorem}[LM~\cite{dwork2016calibrating}]\label{thm:lm}
Given any function $f:D\to \bR^d$ where $D$ is a domain of datasets, the \ac{lm} defined as :
\begin{equation}\label{eq:lm}
    \cM^{\rm Laplace}\left(f(x),\epsilon\right)=f(x)+
    {\left[z_1,\dots,z_d\right]}^T,
\end{equation}
is $\epsilon$-\ac{ldp}. In \eqref{eq:lm},  $z_i\overset{\acs{iid}}{\sim}\lap\left(0, \Delta f/\epsilon \right)$, i.e., they obey an i.i.d. zero-mean Laplace distribution with scale $\Delta f/\epsilon $, 
where 

    $\Delta f\triangleq\max_{x,y\in D} {||f(x)-f(y)||}_1$.    
 
\end{theorem}

\ac{ldp} mechanisms, such as \ac{lm}, guarantee $\epsilon$-\ac{ldp} for a given query of $\mySet{M}$ in \eqref{eq:LDP}. In \ac{fl}, this amounts for a single model update. As \ac{fl} involves multiple rounds, one has to account for the {\em accumulated} leakage, given by the  composition theorem: 
\begin{theorem}[Composition ~\cite{wang2020comprehensive}]\label{thm:ct}
Let $\cM_i$ be an $\epsilon_i$-\ac{ldp} mechanism on input $v$, and $\cM(v)$ is the sequential composition of $\cM_1(v),...,\cM_m(v)$, then $\cM(v)$ satisfies $\sum_{k=1}^m \epsilon_i$-\ac{ldp}.
\end{theorem}

Theorem \ref{thm:ct} indicates that the privacy leakage of each user in \ac{fl} is accumulated as the training proceeds.

\subsection{Problem Formulation}
\label{ssec:Problem formulation}

Our goal is to design a privacy leakage policy alongside privacy-aware user selection. Formally, we aim to set for every round $t \in \bN$ an algorithm that selects $m=|\mathcal{S}_t|$ users, while setting the privacy leakage budget $\{\epsilon_{k,t}\}_{k \in \mathcal{S}_t}$, 
\textcolor{NewColor}{without requiring any prior knowledge on the distribution of the latency \acp{rv} $\{\tau_k\}$.}
These policies should account for the following considerations:
\begin{enumerate}[label={C\arabic*}]
    \item \label{itm:gen} Optimize the accuracy of the trained $\vtheta$ \eqref{eq:theta_opt_def}.
    \item \label{itm:latency} Minimize the overall latency due to \eqref{eqn: Commlatency}.
    \item \label{itm:privacy} Maintain  $\epsb$-\ac{ldp}, i.e., the overall leakage by each user should not exceed $\epsb$, where $\epsb$ is a pre-defined constant.
   \item \label{itm:complexity} Operate with limited complexity to support real-time implementation in large-scale networks. 
\end{enumerate}

The considerations above are addressed in the subsequent sections. 
We first focus solely on considerations~\ref{itm:gen}-\ref{itm:privacy}, based on which we present \ac{pause} in Section~\ref{sec:Selection}. Subsequently, Section~\ref{sec:ALSA} adapts \ac{pause} to accommodate consideration~\ref{itm:complexity}, yielding SA-\ac{pause}, thereby jointly tackling~\ref{itm:gen}-\ref{itm:complexity}.


\section{Privacy-Aware Active User Selection}
\label{sec:Selection}
This section introduces \ac{pause}. We first formulate its time-varying privacy budget policy and associated reward in Subsection~\ref{ssec:policy}. The resulting user selection algorithm is detailed in Subsection~\ref{ssec:PAUSE intro}, with its regret growth analyzed in Subsection~\ref{ssec:regret analysis}. We conclude with a discussion in Subsection~\ref{ssec:discussion}.

\subsection{Reward and Privacy Policy}
\label{ssec:policy}
The formulation of \ac{pause} relies on two main components: $(i)$ a prefixed round-varying privacy budget; and $(ii)$ a reward holistically accounting for privacy, latency, and generalization.
The {\em privacy policy} is designed to ensure that \ref{itm:privacy} is preserved regardless of the number of iterations each user participated in. \textcolor{NewColor}{Namely, for a given overall privacy leakage $\epsb$, our methodology sets a sequence of round-varying privacy budgets.} Accordingly, we define a sequence $\{\epsilon_i\}$ with $\epsilon_i >0$, satisfying:
\begin{equation}
\label{eqn:epsilonSeq}
    \sum_{i=1}^{\infty} \epsilon_i = \epsb,
\end{equation}
for $\epsb$ finite.
Using the sequence $\{\epsilon_i\}$, the privacy budget of any user at the $i$th time it participates in training the model is set to $\epsilon_i$, and achieved using, e.g., \ac{lm}. This guarantees that  
 \ref{itm:privacy} holds. One candidate setting, which is also used in our experiments, sets  $\epsilon_i = \epsb{(e^{\eta}-1)} e^{- \eta i}$. This guarantees achieving asymptotic leakage of $\epsb$ by the limit of a geometric column. for which \eqref{eqn:epsilonSeq} holds when $\eta>0$.

 The {\em reward} guides the active user selection procedure and is based upon two terms. 
 The first is the {\em privacy reward}, which accounts for the fact that our privacy policy has users introduce more dominant \ac{ppn} each time they participate. The privacy reward assigned to the $k$th user at round $t$ is 
 \begin{equation}\label{eq: privacy reward function}
     p_k(t) \triangleq  1- \frac{\sum_{t=1}^{T_k(t)} \epsilon_t}{\epsb},
 \end{equation}
where $T_k(t)$ is the number of rounds the $k$th user has been selected up to and including the $t$th round, i.e., $T_k(t) \triangleq \sum_{i=1}^t \1(k \in \mathcal{S}_t)$. The privacy reward \eqref{eq: privacy reward function} yields higher values to users who have participated in fewer rounds.

The second term is the  {\em generalization reward}, designed to meet  \ref{itm:gen}. It assigns higher values for users whose data have been underutilized compared to the relative size of their data from the whole available data, $ \frac{|\cD_k|}{|\cD|}$. 
We adopt the generalization reward proposed in~\cite{ami2023client}, which was shown to account for both i.i.d. balanced data and non-i.i.d. imbalanced data cases, and rewards the $k$th user in an $m$-sized group at round  $t$ via the function 
\begin{equation}\label{eq: generalization reward}
    g_k(t) \triangleq \bigg|\frac{m}{{|\cD|}/{|\cD_k|}} - \frac{T_k(t)}{t}\bigg|^\beta \cdot \sign \biggl(\frac{m}{{|\cD|}/{|\cD_k|}} - \frac{T_k(t)}{t}\biggr).
\end{equation}
In \eqref{eq: generalization reward}, $\beta >1$ is a hyper-parameter that adjusts the fuzziness of the function, i.e., higher $\beta$ yields lower absolute value where the other parameters are fixed.
Fig.~\ref{fig:g_function} describes $g_k(\cdot)$ as a function of ${T_k(t)}/{t}$, and illustrates the effect of different $\beta$ values as means of balancing the reward assigned to users that participated much (high ${T_k(t)}/{t}$).

\begin{figure}
    \centering
    \includegraphics[width=1\linewidth]{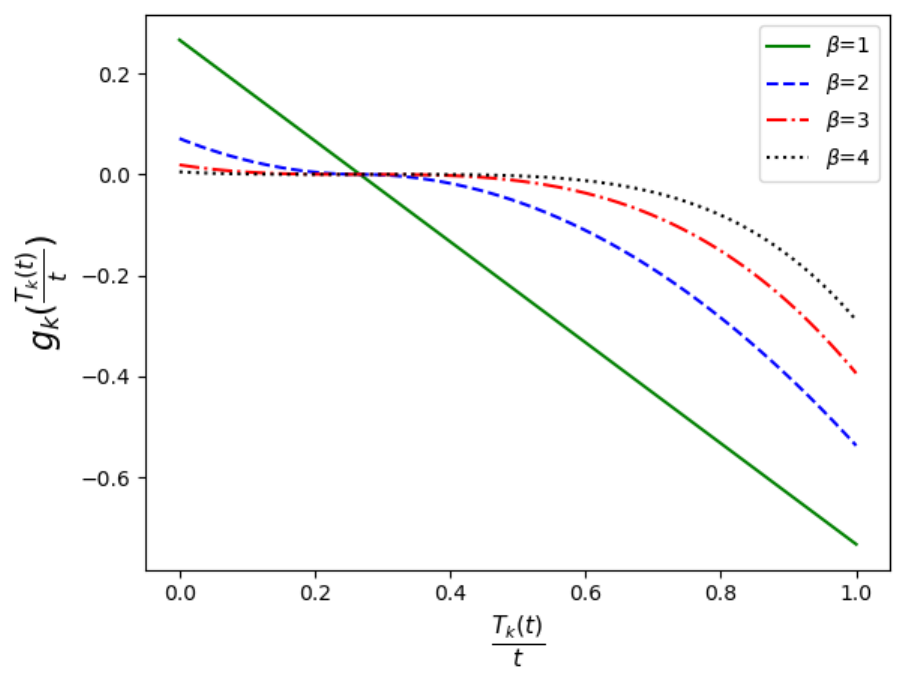}
    \vspace{-0.6cm}
    \caption{Generalization reward \eqref{eq: generalization reward} for different values of $\beta$, with  $\frac{|\cD|}{|\cD_k|} = K$.}
    \label{fig:g_function}
\end{figure}

\color{NewColor}
Our proposed {\em reward}  encompasses the above terms, grading the selection of a group of users $\cS$ of size $m$ at round $t$ as
\begin{align} 
    r(\cS, t) &\triangleq \frac{\tau_{\min}}{\max_{k \in \cS} \tau_{k,t}} \!+\!  \alpha\cdot \Phi_{\rm g}\left(\{g_k(t-1)\}_{k\in \cS}, \textcolor{NewColor2}{\cS} \right)  \notag \\
    &\qquad\qquad+ \gamma \cdot \Phi_{\rm p}\left(\{p_k(t-1)\}_{k\in \cS},  \textcolor{NewColor2}{\cS}\right)  \notag \\ 
    &=\min_{k \in \cS} \frac{\tau_{\min}}{\tau_{k,t}} +  \alpha\cdot \Phi_{\rm g}\left(\{g_k(t-1)\}_{k\in \cS},  \textcolor{NewColor2}{\cS} \right) \notag \\
    &\qquad\qquad + \gamma \cdot \Phi_{\rm p}\left(\{p_k(t-1)\}_{k\in \cS},  \textcolor{NewColor2}{\cS} \right),
    \label{eq: reward func} 
\end{align}
where $\Phi_{\rm g}(\cdot)$ and $\Phi_{\rm p}(\cdot)$ are bounded functions.
\color{black}
The reward in \eqref{eq: reward func} is composed of three additive terms which correspond to~\ref{itm:latency}, \ref{itm:gen} and \ref{itm:privacy}, respectively, with $\alpha$ and $\gamma$ being hyper-parameters balancing these considerations. 
At this point, we can make \textcolor{NewColor}{three} remarks regarding the reward \eqref{eq: reward func}:
\begin{enumerate}[leftmargin=*]

    \item Both $g_k(\cdot)$ and $p_k(\cdot)$ penalize repeated selection of the same users. 
    However, each rewards differently, based on generalization and privacy considerations. The former accounts for the relative dataset sizes of the users, while the latter doesn't.    
    In the  case of homogeneous data, where for all $k \in \bK$, $|\cD_k| = \frac{|\cD|}{K}$, both $g_k(\cdot)$ and $p_k(\cdot)$ play a similar role. However, they differ significantly in the non-i.i.d case.
    \item The value of the first term is determined solely by the slowest user. This non-linearity, combined with the two other terms, directs the algorithm we derive from this reward to select a group of users with similar latency in a given round.\label{itm: non-lineraity in the reawrd}
    \item  \textcolor{NewColor2}{The terms that account for the generalization and privacy are formulated in \eqref{eq: reward func} as the generic bounded functions $\Phi_{\rm g}(\cdot)$ and $\Phi_{\rm p}(\cdot)$. This formulation allows to encompass a broad range of rewards assigned to a selected set of users based on privacy and generalization. For instance, at the system level, federated learning often operates under resource constraints: multiple users may share the same access point, subnet, or geographic region. If too many users from the same cluster are chosen at once, communication can become congested and the diversity of information is reduced \cite{ gafni2022federated} To avoid this, the generalization term can penalize selections that overload a shared resource, encouraging the chosen set of users to be spread across different clusters (as also considered in our numerical study in Section~\ref{sec:Experimental Study}). This helps maintain both communication efficiency and robustness of the model. At the data level, another important factor is class imbalance. If the selected users all contribute data with very similar label distributions, the aggregated model may overfit to certain classes and fail to generalize. The generalization term can instead reward sets of users whose data distributions are complementary, for example by discouraging excessive similarity among their local histograms. In this way, the selection mechanism promotes richer data diversity across the system. We also provide in Subsection~\ref{ssec:special} an analysis of a special case for which these terms are given by averaging functions.}
\end{enumerate}

\subsection{PAUSE Algorithm}
\label{ssec:PAUSE intro}
Here, we present \ac{pause}, which is a \acl{cmab}-based~\cite{chen2013combinatorial} algorithm utilizing the mentioned reward \eqref{eq: reward func}. To derive \ac{pause}, we seek a policy $\Pi \triangleq (\mathcal{S}_1,\mathcal{S}_2,...)$ such that $\E[\sum_{t=1}^n r(\mathcal{S}_t,t)]$ is maximized over  $n$. To maximize the given term, as is customary in \ac{mab} settings, we aim to minimize the {\em regret}, defined as the loss of the algorithm compared to an algorithm composed by a Genie that has prior knowledge of the expectations of the \acp{rv}, i.e., of $\mu_k \triangleq \E[{\tau_{\min}}/{\tau_k}]$. 

We define the Genie's algorithm as  selecting 
\begin{equation}\label{eq: Genie's policy}
    \cG_t \triangleq \argmax_{\cS \subseteq \bK; |\cS| = m} \{C^{\cG}(\cS,t)\},
\end{equation}
where 
\color{NewColor}
\begin{align*}
    C^{\cG}(\cS,t) \triangleq \min_{k \in \cS} \mu_k &+ \alpha\cdot \Phi_{\rm g}\left(\{g_k(t-1)\}_{k\in \cS},  \textcolor{NewColor2}{\cS} \right) \notag \\& + \gamma \cdot \Phi_{\rm p}\left(\{p_k(t-1)\}_{k\in \cS},  \textcolor{NewColor2}{\cS} \right).
\end{align*}
\color{black}
 The Genie policy \eqref{eq: Genie's policy} attempts to maximize the expectation of the reward~\eqref{eq: reward func} in each round, by replacing the order of the expectation and the $\min_{k \in \cS}$ operator.
 As the reward $C^{\cG}$ is history-dependent, the  Genie's policy is history-dependent as well.

We use the Genie policy to derive \ac{pause}, denoted $\cP \triangleq (\cP_1, \cP_2,\ldots)$, as an \ac{ucb}-type algorithm~\cite{auer2002finite}. Accordingly, \ac{pause} estimates the unknown expectations $\{\mu_k\}$ with their empirical means, computed using the latency measured in previous rounds via 
\begin{equation}\label{eq:empirical means}
    \overline{\mu_k}(n) \triangleq \frac{1}{T_k(n)} \sum_{t=1}^n \frac{\tau_{\min}}{\tau_{k,t}} \cdot \1(k \in \cP_t).
\end{equation}
 Note that \eqref{eq:empirical means} can be efficiently updated  in a recursive manner, as 
\begin{equation}\label{eq: ma update rule}
    \overline{\mu_k}(t) = \frac{T_k(t-1)}{T_k(t)}\overline{\mu_k}(t-1) + \frac{\1(k \in \cP_t)}{T_k(t)}\frac{\tau_{min} }{\tau_{k,t}}.
\end{equation}
\ac{pause} uses \eqref{eq:empirical means} to compute the \ac{ucb} terms for each user at the end of the $t^{\rm th}$ round~\cite{auer2002finite}, via 
\begin{equation}\label{eq: ucb definition}
    {\rm ucb}(k,t) \triangleq \overline{\mu_k}(t) + \sqrt{\frac{(m+1)\log(t)}{T_k(t)}}. 
\end{equation}
The \ac{ucb} term in \eqref{eq: ucb definition} is designed to tackle \ref{itm:latency}. Its formulation encapsulates the inherent exploration vs. exploitation trade-off in \ac{mab} problems, boosting exploitation of the fastest users in expectation using $\overline{\mu_k}(t)$, while encouraging exploration of other users in its second term. 
The resulting user selection rule at round $t$ is 
\color{NewColor}
\begin{align} 
    \cP_t =\argmax_{\cS \subseteq \bK; |\cS| = m} \biggr\{&\min_{k \in \cS} {\rm ucb}(k,t-1) \notag \\
   & + \alpha\cdot \Phi_{\rm g}\left(\{g_k(t-1)\}_{k\in \cS},  \textcolor{NewColor2}{\cS} \right) \notag \\& + \gamma \cdot \Phi_{\rm p}\left(\{p_k(t-1)\}_{k\in \cS},  \textcolor{NewColor2}{\cS} \right) \biggl\}.
    \label{eq: pause's policy} 
\end{align}
\color{black}

The overall active user selection procedure is summarized as Algorithm~\ref{alg:Pause}. 
The chosen users send their noisy local model updates to the server, which updates the global model by~\eqref{eq:fl_update} and sends it back to all the users in $\bK$. At the end of every round, we update the users' reward terms for the next round, in which $p_k(t)$ and $\overline{\mu_k}(t)$ change their values only for participating users $k \in \cP_t$.  Note that, by the formulation of Algorithm~\ref{alg:Pause}, it holds that when $m$ is an integer divisor of $K$, then in the first $\frac{K}{m}$ rounds, the server chooses every user exactly once due to the initial conditions.

\begin{algorithm}
    \caption{\ac{pause}}
    \label{alg:Pause}
    \SetKwInOut{Input}{Input} 
    \Input{Set of users $\bK$; Number of active users $m$}   
    \SetKwInOut{Initialization}{Init}
    \Initialization{Set $T_k(0), \overline{\mu_k}(0),p_k(0) \gets 0$; ${\rm ucb}(k,0)\gets\infty$; \newline
    Initial model parameters $\myVec{\theta}_0$}
    {
        \For{$t= 1,2 \ldots$}{%
                    Select $\cP_t$ via \eqref{eq: pause's policy}\; \label{stp:UserSel}

                    Share $\myVec{\theta}_{t-1}$ with users in $\cP_t$\;

                    Aggregate global model $\myVec{\theta}_{t}$ via \eqref{eq:fl_update};

                    \For{$k \in \bK$}{
                        Update $T_k(t) \gets T_k(t-1) + \1(k\in \cP_t)$; 

                        Update empirical estimate $\overline{\mu_k}(t)$ via \eqref{eq: ma update rule};

                        Update  ${\rm ucb}(k,t)$ via~\eqref{eq: ucb definition};

                    }
                    
                    }
        \KwRet{$\myVec{\theta}_t$}
  }
\end{algorithm}

\subsection{Regret Analysis}
\label{ssec:regret analysis}
To evaluate \ac{pause}, we next analyze its {\em regret}, which for a policy $\Pi$ is defined as the expectation of the reward gap between the given policy and the Genie's policy:
\begin{equation}\label{eq: regret definition}
    R^{\Pi}(n) \triangleq \E\Bigl[\sum_{t=1}^n r(\cG_t,t) - r(\mathcal{S}_t,t)\Bigr].
\end{equation}
We define the maximal reward gap for any policy as $\Delta_{\max} \triangleq \max_{t \in \bN, \Pi} r(\cG_t,t) - r(\mathcal{S}_t,t)$. This quantity is bounded as stated the following lemma:
\color{NewColor}
\begin{lemma}
\label{lem:DeltBound}
    User selection via \eqref{eq: pause's policy} with the reward \eqref{eq: reward func} satisfies 
    \begin{equation}
    \label{eqn:DeltBound}
        \Delta_{\max} \leq \max_{k\in \bK} \mu_{k} -\min_{k\in \bK}\mu_{k} + \alpha \Delta_{\Phi_{\rm g}} + \gamma \Delta_{\Phi_{\rm p}},  
    \end{equation}
    where $\Delta_{\Phi_{\rm g}}$ and $\Delta_{\Phi_{\rm p}}$ are the ranges of the bounded $\Phi_{\rm g}$ and $\Phi_{\rm p}$, respectively.
\end{lemma}

\begin{proof}
Inequality \eqref{eqn:DeltBound} follows directly from the boundedness of $\Phi_{\rm g}$ and $\Phi_{\rm p}$ and the structure of \eqref{eq: reward func}.
\end{proof}
\color{black}
We bound the regret of \ac{pause} in the following theorem:
\begin{theorem}\label{THMREGRET}
The regret of \ac{pause} holds
\begin{equation}
        R^{\cP}(n) \leq K(\Delta_{\max} \!+\! \delta) \left(\frac{4(m\!+\! 1)\log(n)}{\delta^2} \!+\! 1\! +\! \frac{2\pi^2}{3}\right),
\end{equation}

\end{theorem}

\begin{proof}
The proof is given in  Appendix~\ref{app:Regret}. 
\end{proof}

\textcolor{NewColor2}{Theorem~\ref{THMREGRET} bounds the regret accumulated at every round $n$. The bound depends linearly on the ranges $\Delta_{\Phi_g}$ and $\Delta_{\Phi_p}$ of the functions $\Phi_g$ and $\Phi_p$ (through Lemma~\ref{lem:DeltBound}), which also allows comparing different formulations of these functions and their associated hyperparameters $(\alpha,\gamma,\rho)$. As is standard in \ac{mab} analysis, the main significance lies in the \emph{growth order} of regret with respect to the number of rounds $n$, rather than in its absolute scale. In the asymptotic regime, \ac{pause} achieves logarithmic regret, i.e., regret growth that does not exceed $\mathcal{O}(\log(n))$.}


\color{NewColor}
\subsection{Special Case: Averaged Generalization and Privacy  Terms}
\label{ssec:special}
The formulation of \ac{pause} in Subsection~\ref{ssec:PAUSE intro} is based on the reward function of \eqref{eq: reward func}, in which the dependencies on the generalization and privacy terms is given in the form of the generic bounded functions   $\Phi_{\rm g}$ and $\Phi_{\rm p}$. While this abstract formulation is amenable to regret analysis as detailed in Subsection~\ref{ssec:regret analysis}, implementing the policy via \eqref{eq: pause's policy} generally requires a computationally exhaustive brute search. However, there exist special cases in which the policy of \ac{pause} can be evaluated with affordable complexity.

To showcase this, we next consider the special case in which the  functions   $\Phi_{\rm g}$ and $\Phi_{\rm p}$ represent averaging, namely, 
\begin{subequations}
    \label{eqn:specialCase}
\begin{align}
    \Phi_{\rm g}\left(\{g_k(t-1)\}_{k\in \cS},  \textcolor{NewColor2}{\cS} \right) &= \frac{1}{m} \sum_{k \in \cS} g_k(t-1),  \\
    \Phi_{\rm p}\left(\{p_k(t-1)\}_{k\in \cS} ,  \textcolor{NewColor2}{\cS}\right) &= \frac{1}{m} \sum_{k \in \cS} p_k(t-1).     \label{eqn:PrivacyLin}   
\end{align}
\end{subequations}
We note that this case preserves the bounded requirement of $ \Phi_{\rm g}$ and $\Phi_{\rm p}$, as $g_k(t) \in [-1,1]$ and $p_k(t)\in[0,1]$ for every $k \in \bK$ and $t\in\bN$, and thus  $\Delta_{\Phi_{\rm g}} =2$ and $\Delta_{\Phi_{\rm p}} = 1$. 

For the special case given by \eqref{eqn:specialCase}, we propose an efficient algorithm for implementing  \eqref{eq: pause's policy} using a heap data structure. The proposed algorithm, termed {\em Pivot-and-Fill}, is summarized as Algorithm~\ref{alg: PivotFill}. There, we omit the round index $t$ from the variables (e.g., use $g_k$ and $p_k$ instead of $g_k(t-1)$ and $p_k(t-1)$) for brevity, and use {\em pop-min}  for popping the minimal element out of the heap, and {\em push} for inserting an element into the heap.

\begin{algorithm}
\caption{Pivot-and-Fill (round $t$)}
\label{alg: PivotFill}
\KwIn{candidate pool $\bK$, set size $m$,
weights $\alpha,\gamma$}
\KwOut{subset $S_t$ of size $m$}

Sort $\bK$ in \emph{descending} order of $\mathrm{ucb}$: $k_1,k_2,\dots,k_K$\;

Initialize a \emph{min-heap} $\mathcal H \leftarrow \{k_1,\dots,k_{m-1}\}$ keyed by $s(\ell)=\alpha g_\ell+\gamma p_\ell$ \;

$\displaystyle G_{\!H}\gets\!\!\!\sum_{\ell\in\mathcal H}\!g_\ell$, $\displaystyle P_{\!H}\gets\!\!\!\sum_{\ell\in\mathcal H}\!p_\ell$\;

$R^\star\gets -\infty$, 
$S^\star\gets\emptyset$\;

\BlankLine

\For{$i=m$ \KwTo $K$}{
    $k\gets k_i$ \tcp*{$k$ is the current pivot}

    $\displaystyle R_k \gets \mathrm{ucb}(k) + \frac{\alpha}{m}(G_{\cH} + g_k) + \frac{\gamma}{m}(P_{\cH}+ p_k)$\;

    \If{$R_k > R^*$}{
    $R^* \gets R_k$, $S^*\gets \cH \cup \{k\}$\;
    }
    
    \If{$\alpha g_k + \gamma p_k > \min\{\cH\}$}
    {
        $(g_\text{out},p_\text{out})\gets$ pop-min$(\mathcal H)$ 
        
        push $(k,\alpha g_k + \gamma p_k)$ into $\mathcal H$ 
        
        $G_{\!H}\gets G_{\!H}-g_\text{out}+g_k$, $P_{\!H}\gets P_{\!H}-p_\text{out}+p_k$\;
    }  
    
} 
$\displaystyle R_K \gets \mathrm{ucb}(k_K) + \frac{\alpha}{m}(G_{\cH} + g_{k_K}) + \frac{\gamma}{m}(P_{\cH}+ p_{k_K})$
\If{$R_K > R^*$}{
$R^* \gets R_K$, $S^*\gets \cH \cup \{k_K\}$\;
}
\Return $S_t\gets S^\star$

\end{algorithm}

\begin{proposition}
    \label{prop:Efficient}
    When $\Phi_{\rm g}$ and $\Phi_{\rm p}$ are given by \eqref{eqn:specialCase}, then Algorithm~\ref{alg: PivotFill} implements \ac{pause}'s policy  \eqref{eq: pause's policy}  with complexity order of $\mathcal{O}(K \log K)$.
\end{proposition}

\begin{proof}
   The complexity order $\mathcal{O}(K \log K)$. arises from sorting a group of size $K$, combined with $K$ iterations of heap operations, each requiring a runtime of order $\mathcal{O}(\log m)$. 
 
 The algorithm's correctness follows from the fact that for any solution of \eqref{eq: pause's policy} under \eqref{eqn:specialCase}, the $ucb$ term is determined by some user $k' \in \bK$. The optimized search in Algorithm~\ref{alg: PivotFill} runs over all options of $k' \in \bK$ and efficiently computes the maximal inner term for each such user. The heap incorporation maintains the $m-1$ users with the largest $\alpha g_k(t-1) + \gamma p_k(t-1)$ among all the users who have a higher or equal ${\rm ucb}$ compared to the $i$th suspected user in the loop. This avoids the re-sorting of lists for every element in the loop, thereby further alleviating the computational complexity. 
\end{proof}

The considered special case thus illustrates that in this particular setting, the formulation of \ac{pause} does not necessarily require computationally intensive brute search and therefore fulfills consideration~\ref{itm:complexity}.

\color{black}

\subsection{Discussion}
\label{ssec:discussion}
\ac{pause} is particularly designed to facilitate privacy and communication constrained \ac{fl}. It leverages \ac{mab}-based active user selection to dynamically cope with privacy leakage accumulation, without restricting the overall number of \ac{fl} rounds as in~\cite{zhang2024dynamic,fraboni2021clustered,chen2024personalized}. \ac{pause} is theoretically shown to achieve best-known regret growth, and it demonstrated promising results in our experiments as detailed in Section~\ref{sec:Experimental Study}.

The formulation of \ac{pause} in Algorithm~\ref{alg:Pause} focuses on the server operation, requiring the users only to send their updates with the proper \ac{ppn}. As such, it can be naturally combined with existing methods for alleviating latency and privacy via update encoding~\cite{gafni2022federated}.  Moreover, the statement of Algorithm~\ref{alg:Pause} complies with any {\em privacy policy} imposed, while adhering to the constraints \ref{itm:gen}-\ref{itm:privacy}. This inherent adaptability makes it an agile solution across diverse policy frameworks.

\textcolor{NewColor}{We note that our latency model assumes independent per-user communication delays drawn from general (user-specific) distributions $\tau_k$. This abstraction provides analytical tractability while capturing user heterogeneity. In practice, however, more complex phenomena such as network congestion, correlated failures, or straggling behavior may arise. These challenges are often addressed in \ac{fl} via deadline-based synchronous schemes~\cite{bonawitz2019towards} or asynchronous FL frameworks~\cite{xie2019asynchronous}, which introduce additional considerations such as model staleness~\cite{ortega2023asynchronous} and partial aggregation~\cite{lang2024stragglers}. While our user selection methodology could potentially be adapted to operate in conjunction with such mechanisms, this extension involves non-trivial modifications and is thus left for future work.}

\textcolor{NewColor}{
The \ac{pause} policy outlined in~\eqref{eq: pause's policy} incorporates two critical hyper-parameters, $\alpha$ and $\gamma$, which emerge from the reward function specified in~\eqref{eq: reward func}. These parameters exert direct influence on user selection due to their additive structure within the reward formulation. Increasing their values drives \ac{pause} toward selecting users who have been relatively underutilized in previous rounds.
The interplay between these parameters reveals distinct optimization priorities, which also depend on the functions $\Phi_{\rm g}$ and $\Phi_{\rm p}$. For instance, under the averaging-based setting in \eqref{eqn:specialCase},  the generalization term accounts for varying data sizes across users, while the privacy term operates independently of dataset magnitude. Consequently, elevating $\alpha$ steers the algorithm toward users who promise maximum learning contribution, potentially at the expense of privacy guarantees. Conversely, increasing $\gamma$ prioritizes privacy-preserving user selection, which inherently introduces additional noise into model updates irrespective of individual user data volumes.
Our empirical evaluation in Section~\ref{sec:Experimental Study} employed case-specific parameter tuning. This manual calibration process highlights an avenue for future research: developing automated hyper-parameter optimization strategies tailored to specific system characteristics and requirements. We leave this study for subsequent investigation.
}

A core challenge associated with applying \ac{pause} in its generic form stems from the fact that~\eqref{eq: pause's policy} involves a brute search over $\binom{K}{m}$ options. Such computation is expected to become infeasible at large networks, i.e., as $K$ grows, making it incompatible with consideration \ref{itm:complexity}. This complexity can be alleviated by approximating the brute search with low-complexity policies based on~\eqref{eq: pause's policy}. \textcolor{NewColor}{Various methods can be considered for tackling the general reward in \eqref{eq: pause's policy} with reduced complexity via heuristic and greedy methods. In the following section, we adopt a method based on \ac{sa}, motivated by the relative simplicity of \ac{sa} and its strong theoretical foundations~\cite{hajek1988cooling}}.


\section{SA-PAUSE}
\label{sec:ALSA}
In this section, we alleviate the computational burden associated with the brute search operation of \ac{pause} \textcolor{NewColor}{in its general formulation as in \eqref{eq: pause's policy}}. The resulting algorithm, termed SA-\ac{pause}, is based on \ac{sa} principles, as detailed in Subsection~\ref{ssec:Annealing}. We analyze SA-\ac{pause}, rigorously identifying conditions for which it coincides with \ac{pause} and characterize its time complexity in Subsection~\ref{ssec:tailored sa Analysis}.

\subsection{Simulated Annealing Algorithm}
\label{ssec:Annealing}
To ease the computational efficiency of the search procedure in~\eqref{eq: pause's policy}, we construct a graph structure where the set of vertices $\bV$ comprises all possible subsets of $m$ users in $\bK$. For each vertex (i.e., set of users) $\mySet{V} \in \bV$, we denote its neighboring set as $\mySet{N}_\mySet{V}$. Two vertices $\mySet{V},\mySet{U} \in \bV$ are designated as neighbors when they satisfy the following requirements:
\begin{enumerate}[label={R\arabic*}:]

\item \label{itm: sa-pause neighboring conditions} The intersection of the vertices contains exactly $m-1$ elements, i.e., the sets of users $\mySet{V}$ and $\mySet{U}$ differ in a single user, thus $|\mySet{V}\cap \mySet{U}|=m-1$.
\textcolor{NewColor}{
\item\label{itm: neighboring condition} One of the users that appears in only a single set minimizes the ${\rm ucb}$ in its designated group. i.e., one of the sets is an {\em active neighbor} of the other. Mathematically, we say that $\mySet{U}$ is an {\em active neighbor} of $\mySet{V}$ (and $\mySet{V}$ is a {\em passive neighbor} of $\mySet{U}$) if the distinct node in $\mySet{V}$, i.e., $k=\mySet{V} \setminus \mySet{U}$, holds 
\begin{align*}
    k = \argmin\limits_{k' \in \mySet{V}} {\rm ucb}(k',t-1).
\end{align*}
}
\end{enumerate}

The above graph construction is inherently undirected due to the symmetric nature of the neighbor relationships.

To formalize our optimization objective, we define the energy of each vertex as the quantity we seek to maximize in \ac{pause}'s search \eqref{eq: pause's policy}. Specifically, for any vertex $\mySet{V}$, define
\textcolor{NewColor}{
\begin{align}
    E(\mySet{V}) \triangleq &\min_{k \in \mySet{V}} {\rm ucb}(k,t-1) + \!\alpha\cdot \Phi_{\rm g}\left(\{g_k(t-1)\}_{k\in \mySet{V}},  \textcolor{NewColor2}{\mySet{V}}\right) \notag
    \\ &+ \gamma \cdot \Phi_{\rm p}(\{p_k(t-1)\}_{k \in \mySet{V}},  \textcolor{NewColor2}{\mySet{V}})). 
    \label{eqn:energy}
\end{align}
}
To identify a vertex exhibiting maximal energy, we introduce an optimized \ac{sa}-based algorithm \cite{hajek1988cooling}, which iteratively inspects vertices (i.e., candidate user sets) in the graph. The resulting procedure, detailed in Algorithm~\ref{alg:tailored-SA}, is comprised of two stages taking place on \ac{fl} round $t$: {\em initialization} and {\em iterative search}.

\smallskip
{\bf Initialization:} 
 Following established \ac{sa} methodology, we maintain an auxiliary temperature sequence, whose $j$th entry is defined as $ \tau_j = \frac{C}{\log(j+1)}$, where parameter $C>0$ exceeds the maximum energy difference between any pair of vertices in the graph. Thus, one must first set the value of $C$. 
\textcolor{NewColor}{
Accordingly, the initialization phase at round $t$ involves sorting all $K$ users according to their respective ${\rm ucb}(k,t-1)$. This is used first to determine an appropriate value for $C$. Denoting the user with the $m$th biggest $\rm{ucb}$ as $k_m$, and following the $\Phi_g$'s and $\Phi_p$'s ranges presented in lemma \ref{lem:DeltBound}, the parameter $C$ is established as follows, where $\omega$ represents a small positive constant:
\begin{align}
    C = & {\rm ucb(k_m, t-1)} - \min_{k' \in \bK} {\rm ucb(k',t-1)} \notag \\
    & + \alpha \cdot \Delta_{\Phi_g} + \gamma \cdot \Delta_{\Phi_p}
    . 
    \label{eq: setting C}
\end{align}
}

\smallskip
{\bf Iterative Search:} 
The algorithm's iterative phase updates an inspected vertex, moving at iteration $j$ from the previously inspected $\mySet{V}_j$ into an updated $\mySet{V}_{j+1}$. This  necessitates the identification of $\mySet{N}_{\mySet{V}_j}$. We decompose this task into the discovery of active and passive neighbors as specified in~\ref{itm: neighboring condition}, utilizing the previously constructed sorted list:
\textcolor{NewColor}{
\begin{enumerate}[label={N\arabic*}:]
\item \textbf{Active Neighbor Identification} - To determine the active neighbors in iteration $i$,
we substitute the user with the minimal ${\rm ucb}(k,t-1)$ in $\mySet{V}_j$ by a user that isn't in the mentioned set. This procedure yields at most $K-m$ active neighbors of $\mySet{V}$.\label{itm: finding active neighbors}
\item \textbf{Passive Neighbor Identification} - For passive neighbors, we establish that a vertex $\mySet{U}$ qualifies as a passive neighbor of $\mySet{V}_j$ if it can be constructed through one of two mechanisms. Let $a$ denote the user with minimal ${\rm ucb}(k,t-1)$ in $\mySet{V}_j$ and $b$ represent the user with the second-minimal value. $\mySet{U}$ is a passive neighbor of $\mySet{V}_j$ if it is obtained by either:\label{itm: finding passive neighbors}
\begin{enumerate}
\item Replace any user in $\mySet{V}_j$ except $a$ with a user whose ${\rm ucb}(k,t-1)$ value is lower than $a$'s (positioned before $a$ in the sorted list).
\item Replace $a$ with a user whose ${\rm ucb}(k,t-1)$ value is lower than $b$'s (positioned before $b$ in the sorted list).
\end{enumerate}
\end{enumerate}
}

Once the neighbors set $\mySet{N}_{\mySet{V}_j}$ is formulated, the algorithm inspects a random neighbor $\mySet{U}$. This set is inspected in the following iteration if it improves in terms of the energy \eqref{eqn:energy} (for which it is also saved as the best set explored so far), or alternatively it is randomly selected with probability $ \exp{\big(-\frac{E(\mySet{U}) - E(\mySet{V}_j)}{\tau_j}\big)}$. The resulting procedure is summarized as Algorithm~\ref{alg:tailored-SA}.

\begin{algorithm}
    \caption{Tailored \ac{sa} for \ac{pause} at round $t$}
    \label{alg:tailored-SA}
    \SetKwInOut{Input}{Input} 
    \Input{Set of users $\bK$;  Number of active users $m$}   
    \SetKwInOut{Initialization}{Init}
    \Initialization{Randomly sample a vertex $\mySet{V}_1$ and set $\mySet{P}_t = \mySet{V}_1$; \newline Sort the users along ${\rm ucb}(k,t-1)$.}
    {
    Compute $C$ via \eqref{eq: setting C}\;
        \For{$j= 1,2 \ldots$}{%
                Find $N_{\mySet{V}_j}$ as described in \ref{itm: finding active neighbors} and \ref{itm: finding passive neighbors};

                    Sample randomly $\mySet{U} \in \mySet{N}_{\mySet{V}_j}$;

                    \If{$E(\mySet{U}) \geq E(\mySet{V}_j)$}{
                    Update inspected vertex $\mySet{V}_{j+1} \gets \mySet{U}$\;
                    Update best vertex $\mySet{P}_t \gets \mySet{U}$\;
                    }
                    \Else{Sample $p$ uniformly over $[0,1]$\;
                    Set $\tau_j = \frac{C}{\log(1+j)}$\;                    
                    \If{$p \leq \exp{\big(-\frac{E(\mySet{U}) - E(\mySet{V}_j)}{\tau_j}\big)}$}{  Update inspected vertex $\mySet{V}_{j+1} \gets \mySet{U}$\;}
                    \Else{Re-inspect vertex $\mySet{V}_{j+1} \gets \mySet{V}_j$\;}

  }
  }
  \KwRet{$\mySet{P}_t$}
  }
\end{algorithm}

The proposed SA-\ac{pause} implements its \ac{fl} procedure with active user selection formulated, while using Algorithm~\ref{alg:tailored-SA} to approximate \ac{pause}'s search~\ref{eq: pause's policy}. SA-\ac{pause} thus realizes Algorithm~\ref{alg:Pause} while replacing its Step~\ref{stp:UserSel} with Algorithm~\ref{alg:tailored-SA}.

\subsection{Theoretical Analysis}
\label{ssec:tailored sa Analysis}
{\bf Optimality:}
The \ac{sa} search of \ac{sa}-\ac{pause}, detailed in Algorithm~\ref{alg:tailored-SA}, replaces searching over all possible user selections with exploration over a graph. To show its validity, we first prove that it indeed finds the reward-maximizing set of users, as done in \ac{pause}. Since in general there may be more than one set of users that maximizes the reward (or equivalently, the energy  \eqref{eqn:energy}), we use $\cJ$ to denote the set of vertices exhibiting maximal energy in the graph. 
The ability of Algorithm~\ref{alg:tailored-SA} to recover the same users set as brute search via \eqref{eq: pause's policy} (or one that is equivalent in terms of reward) is stated in the following theorem:
 \begin{theorem}\label{THMSA}
For Algorithm~\ref{alg:tailored-SA}, it holds that:
     \begin{equation}
     \lim_{j \rightarrow \infty} \bP(\mySet{V}_j \in \cJ) = 1.
    \end{equation}
 \end{theorem}
 
\begin{proof}
The proof is given in  Appendix~\ref{app:sa convergence}. 
\end{proof}

 Theorem~\ref{THMSA} shows that Algorithm~\ref{alg:tailored-SA} is guaranteed to recover the reward-maximizing users set in the horizon of an infinite number of iterations. 
 While the \ac{sa} algorithm operates over a finite number of iterations, and Theorem \ref{THMSA} applies as $j \rightarrow \infty$, the carefully designed cooling temperature sequence and algorithmic structure ensure robust practical performance of \ac{sa} algorithms \cite{henderson2003theory, ledesma2008practical}. This efficacy is empirically validated in Section~\ref{sec:Experimental Study}.
 
\smallskip
\noindent
{\bf Time-Complexity:}
Having shown that Algorithm~\ref{alg:tailored-SA} can approach the users' set recovered via \ac{pause}, we next show that it satisfies its core motivation, i.e., carry out this computation with reduced complexity, and thus supports scalability. While inherently the number of selected users $m$ is smaller than the overall number of users $K$, and often $m \ll K$, we accommodate in our analysis computationally intensive settings where $m$ is allowed to grow with $K$, but in the order of $m=\Theta(K)$.

On each \ac{fl} round $t$, the initialization phase requires $\mySet{O}(K\log K)$ operations due to the list sorting procedures. During each iteration $j$, locating $\mySet{V}_j$'s users' indices in the sorted lists can be accomplished in $\mySet{O}(K\log K)$ operations through pointer manipulation. The identification of $\mySet{N}_{\mySet{V}_j}$ exhibits complexity $\mySet{O}(|\mySet{N}_{\mySet{V}_j}|)$, as each neighbor can be found in constant time. While the number of active neighbors is bounded by $K-m$, the quantity of passive neighbors varies across users and iterations.
Given that each passive neighbor of $\mySet{V}_j$ corresponds to that node being an active neighbor of $\mySet{V}_j$, and considering the bounded number of active neighbors per user, a balanced graph typically exhibits approximately $K-m$ passive neighbors per user. Specifically, in the average case where each user in $\bV$ has $\mySet{O}(K\log K)$ passive neighbors, the complexity order of Algorithm~\ref{alg:tailored-SA} is  $\mySet{O}(K\log K)$.

For comparative purposes, consider a simplified \ac{sa}  variant (termed {\em Vanilla-SA}) where the neighboring criterion is reduced to only the first condition in \ref{itm: sa-pause neighboring conditions} (i.e., nodes are neighbors if they share exactly $m-1$ users). This algorithm closely resembles Algorithm~\ref{alg:tailored-SA}, but eliminates list sorting and determines $N_{\mySet{V}_j}$ by exhaustively replacing each user in $\mySet{V}_j$ with each user in $\bK \setminus \{\mySet{V}_j\}$.
In this case, by setting $C$ to be an upper bound on $\Delta_{max}$~\eqref{eqn:DeltBound}, e.g., $C \triangleq 2\alpha + \gamma +1$ we satisfy the conditions for Theorem~\ref{THMSA} as well, ensuring asymptotic convergence. However, this approach results in $|N_{\mySet{V}_j}| = m(K-m)$, producing a densely connected graph that impedes search efficiency and invariably yields $\mySet{O}(K^2)$ complexity. Table \ref{tab: complexity} presents a comprehensive comparison of time complexities across different scenarios.

\begin{table}
   \centering
   \begin{tabular}{|c|c|c|c|}
       \hline
       \renewcommand{\arraystretch}{1.5}  
       \diagbox{Algorithm}{Case} & Best & Average & Worst \\
       \hline
       Brute force search \ref{eq: pause's policy} & \multicolumn{3}{c|}{$O(e^K)$} \\
       \hline
       Vanilla-SA  & \multicolumn{3}{c|}{$O(K^2)$} \\
       \hline
       Algorithm~\ref{alg:tailored-SA} & \multicolumn{2}{c|}{$O(K\log K)$} & $O(K^2)$ \\
       \hline
   \end{tabular}
   \vspace{0.2cm}
   \caption{Time complexity comparison of different algorithms}
   \label{tab: complexity}
   \vspace{-0.6cm}
\end{table}

\smallskip
\noindent
{\bf Summary:}
Combining the optimality analysis in Theorem~\ref{THMSA} with the complexity characterization in Table~\ref{tab: complexity} indicates that the integration of Algorithm~\ref{alg:tailored-SA} to approximate \ac{pause}'s search~\eqref{eq: pause's policy} into SA-PAUSE enables the application of \ac{pause} to large-scale networks, meeting \ref{itm:complexity}. The theoretical convergence guarantees, coupled with its practical efficiency, make it a robust solution for approximating \ac{pause} and thus still adhering to considerations \ref{itm:gen}-\ref{itm:privacy}. The empirical validation of these theoretical results is presented comprehensively in the following section.


\section{Numerical Study}
\label{sec:Experimental Study}
\subsection{Experimental Setup}
\label{ssec:setup}
Here, we numerically evaluate \ac{pause} in \ac{fl}\footnote{The source code used in our experimental study, including all the hyper-parameters, is available online at \url{https://github.com/oritalp/PAUSE}}.  
We consider the training of a \ac{dnn}  for image classification based on MNIST and CIFAR-10, \textcolor{NewColor}{which are widely-used for empirical evaluation of client selection in \ac{fl}, representing non-trivial tasks that can be tackled with different \ac{dnn} architectures and learning frameworks~\cite{mayhoub2024review}}.
\textcolor{NewColor}{
We train three different \ac{dnn} architectures: 
$(i)$ a  three-layer \ac{fc} network with 32 neurons at its widest layer for MNIST; 
$(ii)$ a \ac{cnn} with three hidden layers followed by a \ac{fc} network with two hidden layers for CIFAR-10; and 
$(iii)$  a larger \ac{cnn} with five hidden layers followed by a \ac{fc} network with two hidden layers for CIFAR-10 under the {\em large network} setting in Subsection~\ref{ssec:LargeNet}.}

We examine our approach 
in both small and large network settings with varying privacy budgets. In the former, the data is divided between $K=30$ users, and $m=5$ of them are chosen at each round, while the latter corresponds to $K=300$ and $m=15$ users. The communication latency  $\tau_k$ obeys a normal distribution for every $k \in \bK$. The users are equally divided into two groups: fast users, who had lower communication latency expectations, and slower users.
For each configuration, we test our approach both in i.i.d. and non-i.i.d. data distributions. In the imbalanced case, the data quantities are sampled from a Dirichlet distribution with parameter $\bm{\alpha}$, where each user exhibits a dominant label comprising approximately a quarter of the data.

\textcolor{NewColor2}{
We focus in this study on two different reward formulations:
\begin{enumerate}[label={R\arabic*}]
    \item \label{itm:average} A loss formulation which takes the form as in \eqref{eqn:specialCase}. Evaluating our framework for this case allows us to assess \ac{pause} using Algorithm~\ref{alg: PivotFill} and to compare SA-PAUSE to the exact solution in the large network case as well.
    \item \label{itm:Gen} A non-separable reward representing a network access constraint. Here, the users are randomly assigned into a set of $R$ distinct clusters $\{\mySet{C}_r\}_{r=1}^R$, and the generalization term encourages selecting users from different clusters, i.e., 
    \begin{align}
        & \Phi_{\rm g}\left(\{g_k(t)\}_{k\in \cS}, \cS\right)  = \frac{1}{m} \sum_{k \in \cS} g_k(t) \notag \\
      & \qquad \qquad - \rho \sum_{r=1}^R \max\left(0, \left|\cS \cap \mySet{C}_r\right| - 1\right),
         \label{eqn:GenLoss}
    \end{align}
    while the privacy term is set via \eqref{eqn:PrivacyLin}.
    To capture the network access consideration in the performance, we add to the reported per-round latency an additional penalty of $\delta_{\tau} \cdot \sum_{r=1}^R \max\left(0, \left|\cS \cap \mySet{C}_r\right| - 1\right)$. For the small network we set  $R=6$ and $\delta_{\tau}= 0.05$, while for the large network we use   $R=20$ and $\delta_{\tau}= 0.01$. In the latter, \ac{pause} becomes computationally infeasible for large networks, and one must resort to SA-PAUSE.
\end{enumerate} 
}
 Our algorithms are compared with the following benchmarks: 
 \begin{itemize}[leftmargin=*]
     \item {\em Random}, uniformly sampling $m=5$ users without replacements~\cite{li2019convergence}, solely in the i.i.d balanced case.
\item {\em \ac{fa} with privacy} and {\em FedAvg w.o. privacy}, choosing all $K$ users, with and without privacy, respectively.
\item  {\em Fastest in expectation}, using only the same pre-known five fastest $m$ users in expectation at each round. 
\item The {\em clustered sampling}  selection algorithm proposed in \cite{fraboni2021clustered}.
 \end{itemize}

\color{NewColor2}
\subsection{Small Network with i.i.d. Data}
Our first study trains the mentioned \ac{cnn} \textcolor{NewColor}{with 3 hidden layers} using an overall privacy budget of $\epsb=40$ for image classification using the CIFAR-10 dataset. The resulting \ac{fl} accuracies versus communication latency are illustrated in Fig.~\ref{fig: CIFAR-10 30 choose 5 i.i.d data, val acc vs. time} under \ref{itm:average} and in Fig.~\ref{fig: rev2_constrained_small_acc} for \ref{itm:Gen}. The error curves were smoothened with an averaging window of size $10$ to attenuate the fluctuations. As expected, due to privacy leakage accumulation, the more rounds a user participates in, the noisier their updates are. This is evident in both Figs.~\ref{fig: CIFAR-10 30 choose 5 i.i.d data, val acc vs. time}-\ref{fig: rev2_constrained_small_acc}, where choosing all users quickly results in ineffective updates. 
\ac{pause} consistently achieves both accurate learning and rapid convergence. Further observing this figure indicates SA-PAUSE successfully approximates \ac{pause}'s brute force search as well.

\begin{figure}
    \centering
    \includegraphics[width=1\linewidth, height=5.5cm]{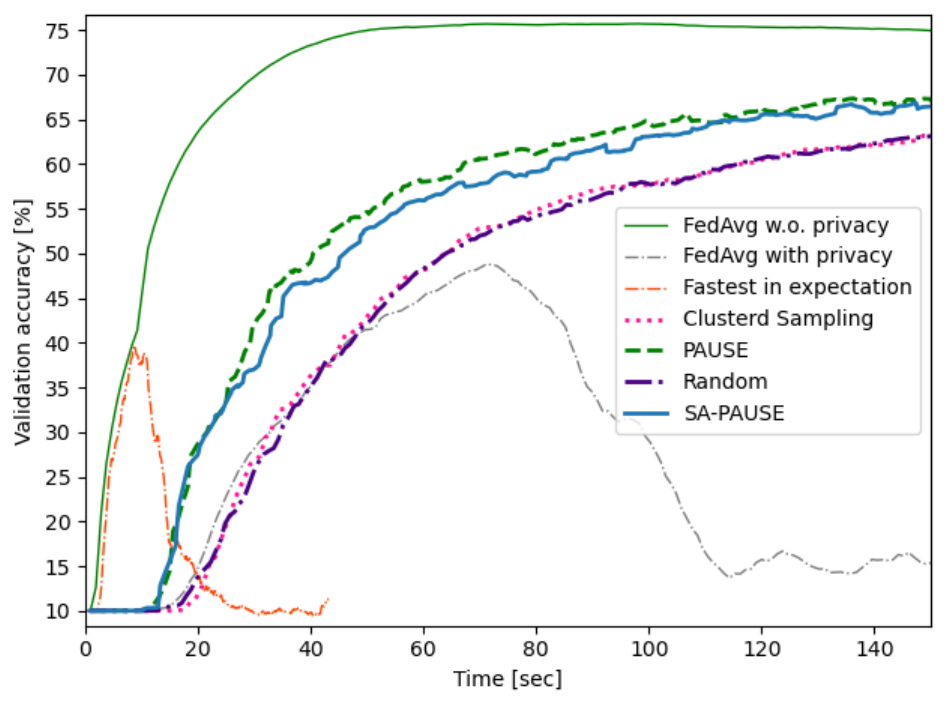}
    \vspace{-0.6cm}
    \caption{Validation accuracy vs. latency, 3 layers CNN trained on CIFAR-10, i.i.d. data, small network, reward \ref{itm:average}}
    \label{fig: CIFAR-10 30 choose 5 i.i.d data, val acc vs. time}
    \vspace{-0.4cm}
\end{figure}
\begin{figure}
    \centering
    \includegraphics[width=1\linewidth, height=5.5cm]{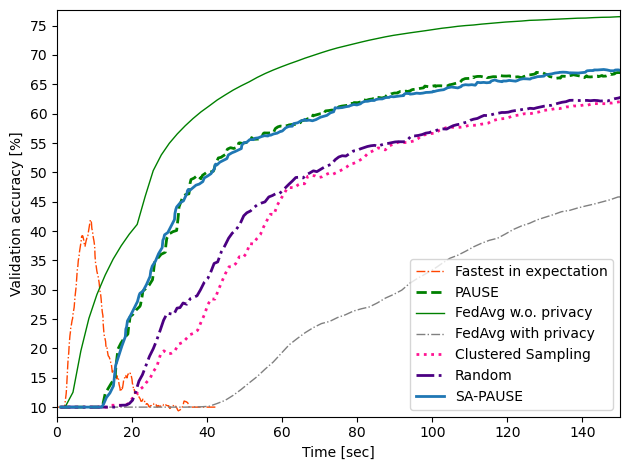}
    \vspace{-0.6cm}
    \caption{Validation accuracy vs. latency, 3 layers CNN trained on CIFAR-10, i.i.d. data, small network, reward \ref{itm:Gen}}
    \label{fig: rev2_constrained_small_acc}
    \vspace{-0.2cm}
\end{figure}

\ac{pause}'s ability to mitigate privacy accumulation is showcased in Figs.~\ref{fig: CIFAR iid 30 choose 5 privacy vs. epochs}-\ref{fig: rev2_constrained_small_privacy}. There, we report the overall leakage as it evolves over epochs under \ref{itm:average}-\ref{itm:Gen}, respectively. 
Fig.~\ref{fig: CIFAR iid 30 choose 5 privacy vs. epochs} reveals that the privacy violation at each given epoch using \ac{pause} is lower compared to the random and the clustered sampling methods, adding to its improved accuracy and latency noted in Fig.~\ref{fig: CIFAR-10 30 choose 5 i.i.d data, val acc vs. time}. 
Comparing Fig.~\ref{fig: rev2_constrained_small_privacy} with Fig.~\ref{fig: CIFAR iid 30 choose 5 privacy vs. epochs}, one could spot that the incorporation of the additional reward consideration in \eqref{eqn:GenLoss} leads to mild decrease in the privacy leakage management of \ac{pause} and its approximation, though yet being superior to the compared algorithms.
Note that {\em FedAvg with privacy} and {\em fastest in expectation} methods' maximum privacy violation coincide, as in every round it is raised by an $\epsilon_i$.

\begin{figure}[t]
    \centering
    \includegraphics[width=1\linewidth, height=5.5cm]{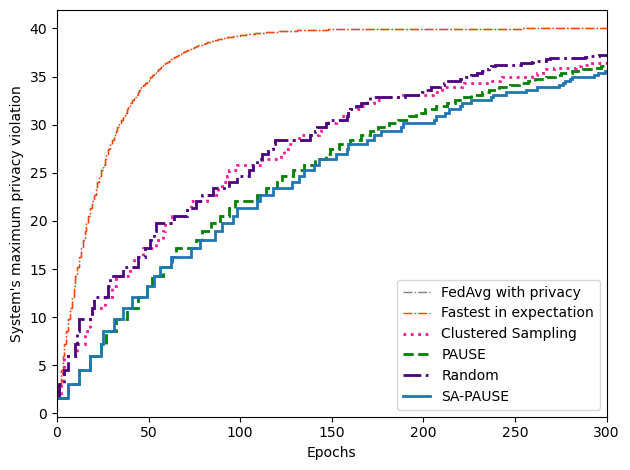}
       \vspace{-0.6cm}
    \caption{Privacy leakage vs. global epochs, 3 layers CNN trained on CIFAR-10. i.i.d. data,  small network, reward \ref{itm:average}}
    \label{fig: CIFAR iid 30 choose 5 privacy vs. epochs}
    \vspace{-0.4cm}
\end{figure}

\begin{figure}
    \centering
    \includegraphics[width=1\linewidth, height=5.5cm]{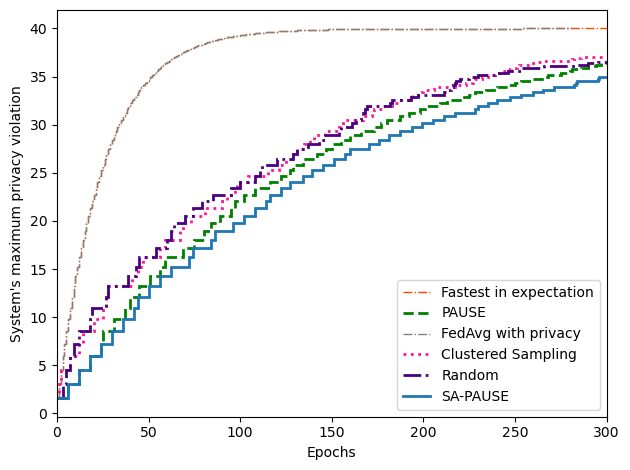}
    \vspace{-0.6cm}
    \caption{Privacy leakage vs. global epochs, 3 layers CNN trained on CIFAR-10, i.i.d. data, small network, reward reward \ref{itm:Gen}}
    \label{fig: rev2_constrained_small_privacy}
    \vspace{-0.2cm}
\end{figure}

\color{black}
\subsection{Small Network with non-i.i.d. Data}
\label{ssec:SmalNoniid}
Subsequently, we train the same \ac{dnn} with CIFAR-10 in the non-i.i.d case as described previously with an overall privacy budget of $\epsb=100$ \textcolor{NewColor2}{under the reward in \ref{itm:average}}. As opposed to the balanced data test, this setting necessitates balancing between users with varying quantities of data, which might contribute differently to the learning process. The data quantities were sampled from a Dirichlet distribution with parameter ${\bm\alpha} = {\bm 3}$.
Analyzing the validation accuracy versus communication latency in Fig.~\ref{fig: CIFAR_small_non_iid_val_acc} indicates the superiority of our algorithms also in this case in terms of accuracy and latency. 
Fig.~\ref{fig: CIFAR non-iid 30 choose 5 privacy vs. time} depicts the maximum privacy violation of the system, this time, versus the communication latency, and facilitates this statement by demonstrating both \ac{pause} and its approximation's ability to maintain privacy better, although performing more sever client iterations in any given time. \textcolor{NewColor}{As in the preceding studies, we consistently observe the ability of the \ac{sa}-based algorithm to approach the direct computation of \ac{pause} via \eqref{eq: pause's policy}.} 

\begin{figure}[t]
    \centering
    \includegraphics[width=1\linewidth, height=5.5cm]{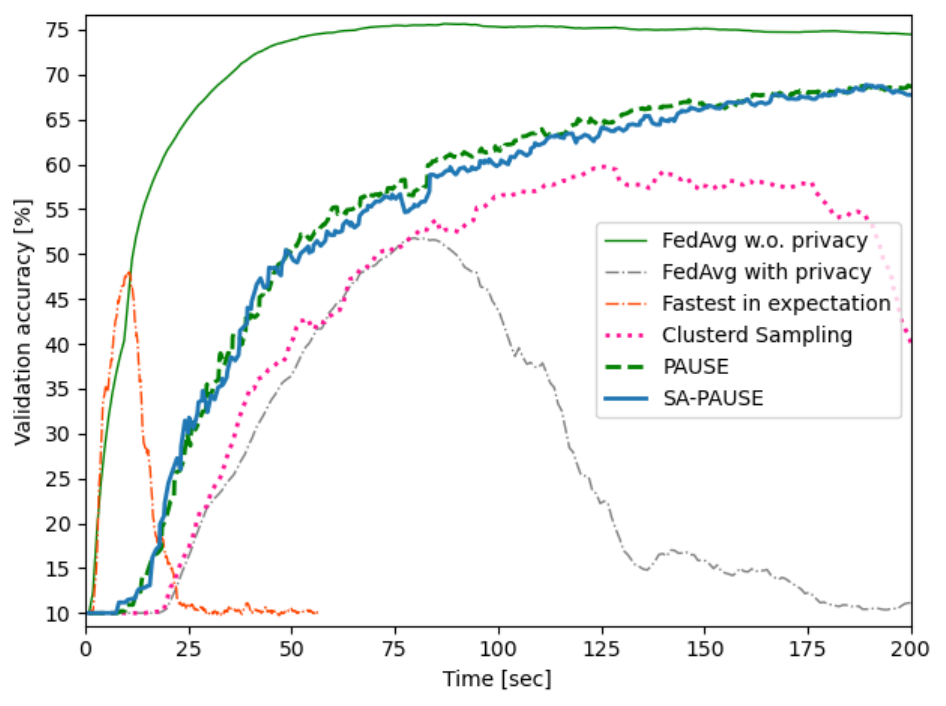}
       \vspace{-0.6cm}
    \caption{Validation accuracy vs. latency, 3 layers CNN trained on CIFAR-10, non-i.i.d data, small network, reward \ref{itm:average}}
    \label{fig: CIFAR_small_non_iid_val_acc}
    \vspace{-0.4cm}
\end{figure}

\begin{figure}
    \centering
    \includegraphics[width=1\linewidth, height=5.5cm]{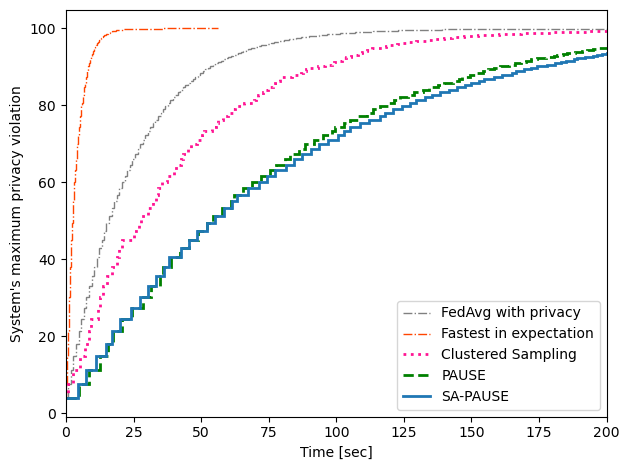}
       \vspace{-0.6cm}
    \caption{Privacy leakage vs.  latency, 3 layers CNN trained on CIFAR-10, non-i.i.d data,  small network, reward \ref{itm:average}}
    \label{fig: CIFAR non-iid 30 choose 5 privacy vs. time}
    \vspace{-0.2cm}
\end{figure}

\subsection{Large Networks}
\label{ssec:LargeNet}
We proceed to consider the large network settings. Here, we train {three models: one for MNIST with i.i.d. data distribution, and both mentioned \ac{cnn}s for CIFAR-10 with  non-i.i.d. data distribution.
}
\textcolor{NewColor2}{User selection for all three models is based on reward \ref{itm:average}, while the three-layered \ac{cnn} is also trained when using \ref{itm:Gen}.}
For these scenarios, we implemented two modifications. First, to accelerate the convergence of the \ac{sa} procedure in Algorithm~\ref{alg:tailored-SA} under a reasonable number of iterations, we modulate the temperature coefficient $C$ as in ~\cite{ben2004computing, bezakova2008accelerating}. This is accomplished by dividing the temperature coefficient by a constant $\kappa=30$, i.e., the temperature in the $j$th iteration becomes $\tau_j = \frac{C}{\kappa \log(1+j)}$~\cite{ben2004computing, bezakova2008accelerating}. Second, to enhance exploitation~\cite{wu2018adaptive,drugan2014pareto},  we amplified the empirical mean $\overline{\mu_k}(t)$ in~\ref{eq: ucb definition} by another constant, $\zeta=3$.

{The overall privacy budget for the MNIST experiment was set to $ \epsb=10$. In contrast, the CNNs trained on CIFAR-10 had privacy budgets of $\epsb=10$ for the 3-layer CNN and $\epsb=15$ for the larger neural network.}
The data quantities were sampled from a Dirichlet distribution with parameter ${\bm\alpha} = {\bm 2}$ in the first case, and with ${\bm\alpha} = {\bm 3}$ in the subsequent cases. All cases exhibited consistent trends with the small networks tests, systematically demonstrating SA-PAUSE's robustness across diverse privacy budgets, datasets, and network scales.

As before, we present the validation accuracy versus communication latency alongside the maximum overall privacy leakage versus time. These results are presented in Figs.~\ref{fig: MNIST_large_iid_acc}-\ref{fig: MNIST_large_iid_privacy} for MNIST; in Figs.~\ref{fig: CIFAR_large_non_iid_val_acc}-\ref{fig: rev2_constrained_large_privacy} for CIFAR-10 with the small \ac{cnn}; and in Figs.~\ref{fig: CIFAR_cnn5_acc_paper}-\ref{fig: CIFAR_cnn5_privacy_paper} for CIFAR-10 with the larger \ac{cnn}. 
{These results systematically demonstrate the ability of our proposed \ac{sa}-\ac{pause} to facilitate rapid learning with balanced and limited privacy leakage, not only over large networks but also on deeper neural network architectures.} 
\textcolor{NewColor2}{Particularly, comparing the performance achieved with the reward \eqref{itm:average} to the one in \ref{itm:Gen}, we note that the additional network constraints encapsulated in \ref{itm:Gen} affect convergence, especially in its early stages. }

\begin{figure}[t]
    \centering
    \includegraphics[width=1\linewidth, height=5.5cm]{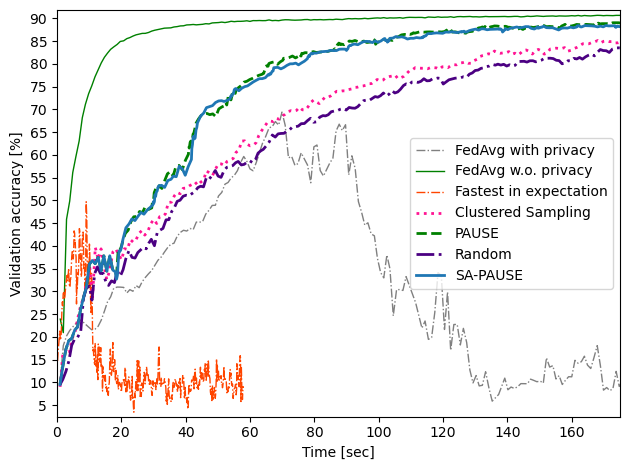}
       \vspace{-0.6cm}
    \caption{Validation accuracy vs. latency, MNIST, i.i.d data, large network, reward \ref{itm:average}}
    \label{fig: MNIST_large_iid_acc}
    \vspace{-0.4cm}
\end{figure}

\begin{figure}
    \centering
    \includegraphics[width=1\linewidth, height=5.5cm]{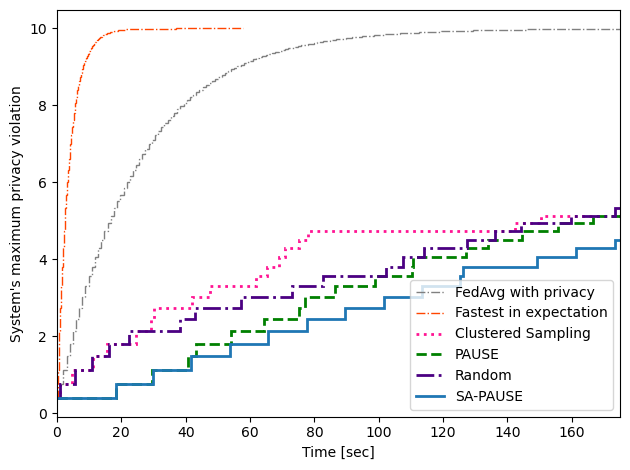}
       \vspace{-0.6cm}
    \caption{Privacy leakage vs. latency, MNIST, i.i.d data, large network, reward \ref{itm:average}}
    \label{fig: MNIST_large_iid_privacy}
    \vspace{-0.2cm}
\end{figure}


\begin{figure}[t]
    \centering
    \includegraphics[width=1\linewidth, height=5.5cm]{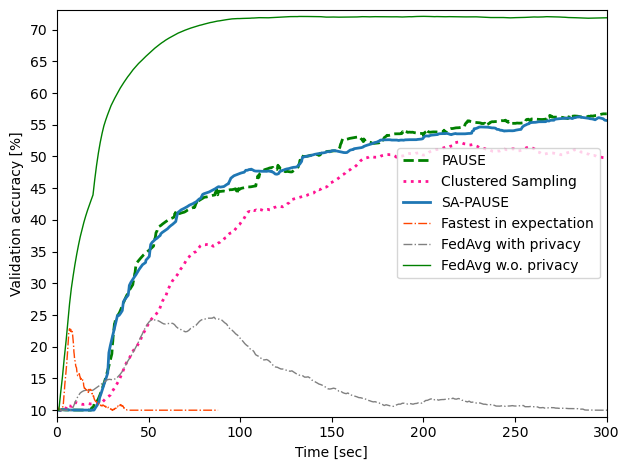}
       \vspace{-0.6cm}
    \caption{Validation accuracy vs. latency, 3 layers CNN trained on CIFAR-10, non-i.i.d data, large network, reward \ref{itm:average}}
    \label{fig: CIFAR_large_non_iid_val_acc}
    \vspace{-0.4cm}
\end{figure}

\begin{figure}
    \centering
    \includegraphics[width=1\linewidth, height=5.5cm]{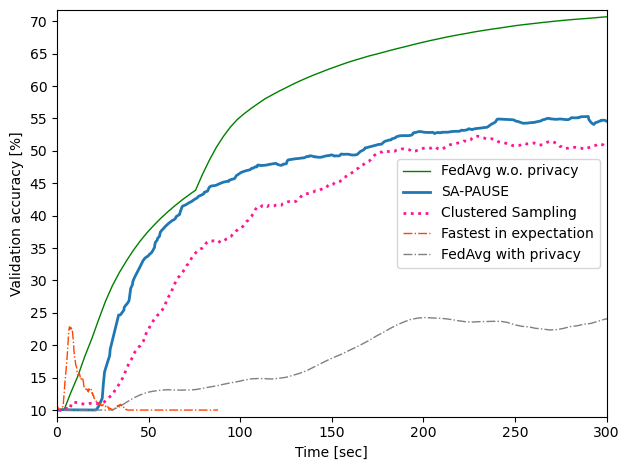}
    \vspace{-0.6cm}
    \caption{Validation accuracy vs. latency, 3 layers CNN trained on CIFAR-10, non-i.i.d data, large network, reward \ref{itm:Gen}}
    \label{fig: rev2_constrained_large_acc}
    \vspace{-0.2cm}
\end{figure}

\begin{figure}[t]
    \centering
    \includegraphics[width=1\linewidth, height=5.5cm]{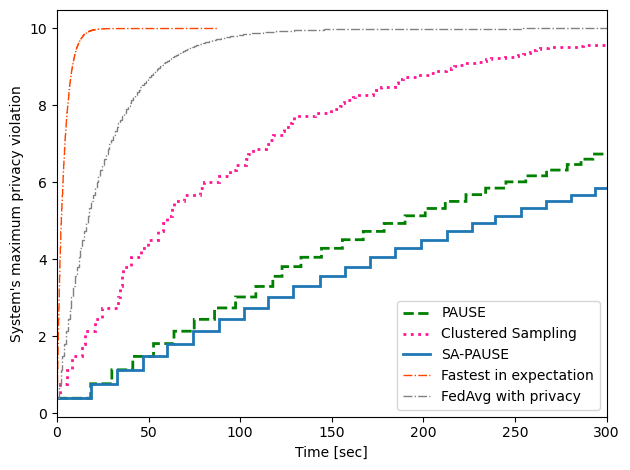}
       \vspace{-0.6cm}
    \caption{Privacy leakage vs.  latency, 3 layers CNN trained on CIFAR-10, non-i.i.d data, large network,  reward \ref{itm:average}}
    \label{fig: CIFAR_large_non_iid_privacy}
    \vspace{-0.2cm}
\end{figure}

\begin{figure}
    \centering
    \includegraphics[width=1\linewidth, height=5.5cm]{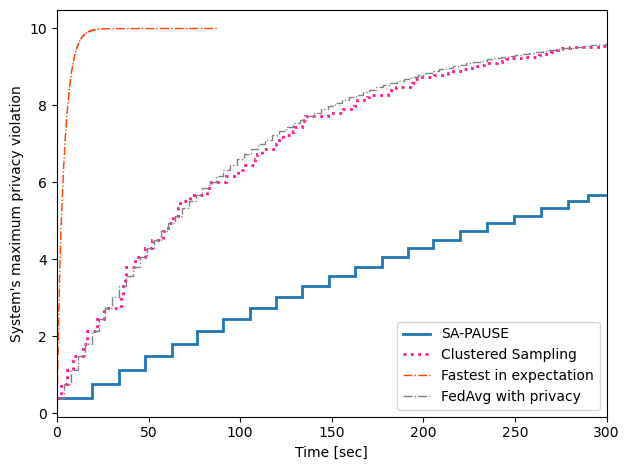}
    \vspace{-0.6cm}
    \caption{Privacy leakage vs. latency, 3 layers CNN trained on CIFAR-10, non-i.i.d data, large network, reward \ref{itm:Gen}}
    \label{fig: rev2_constrained_large_privacy}
    \vspace{-0.2cm}
\end{figure}


\begin{figure}[t]
    \centering
    \includegraphics[width=1\linewidth, height=5.5cm]{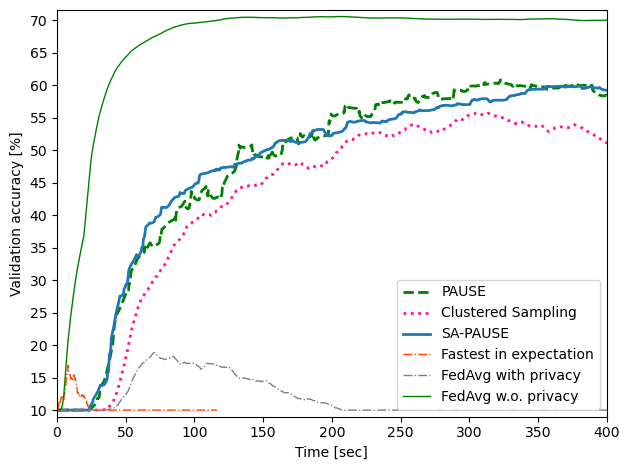}
       \vspace{-0.6cm}
    \caption{Validation accuracy vs. latency, 5 layers CNN trained on CIFAR-10, non-i.i.d data, large network, reward \ref{itm:average}}
    \label{fig: CIFAR_cnn5_acc_paper}
    \vspace{-0.4cm}
\end{figure}

\begin{figure}
    \centering
    \includegraphics[width=1\linewidth, height=5.5cm]{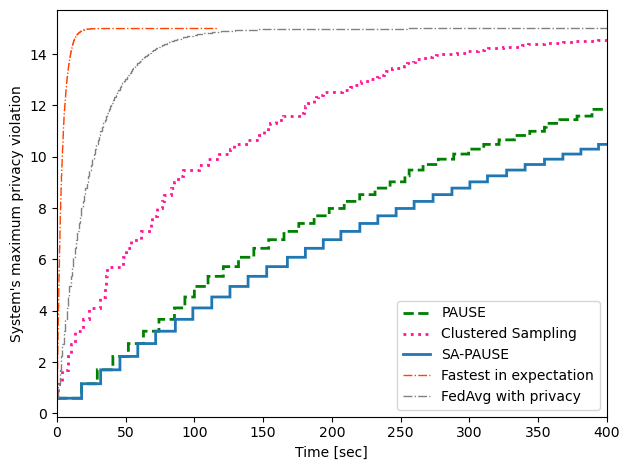}
       \vspace{-0.6cm}
    \caption{Privacy leakage vs.  latency, 5 layers CNN trained on CIFAR-10, non-i.i.d data, large network,  reward \ref{itm:average}}
    \label{fig: CIFAR_cnn5_privacy_paper}
    \vspace{-0.2cm}
\end{figure}

\color{black}

\section{Conclusion}
\label{sec:conclusions}
We proposed \ac{pause}, an active and dynamic user selection algorithm under fixed privacy constraints. This algorithm balances three \ac{fl} aspects: accuracy of the trained model, communication latency, and the system's privacy. We showed that under common assumptions, \ac{pause}'s regret achieves a logarithmic order with time. To address complexity and scalability, we developed SA-PAUSE, which integrates a \ac{sa} algorithm with theoretical guarantees to approximate \ac{pause}'s brute force search in feasible running time. We numerically demonstrated SA-PAUSE's ability to approximate PAUSE's search and its superiority over alternative approaches in diverse experimental scenarios.

\ifFullVersion
\begin{appendix}
		\numberwithin{equation}{subsection}	
\numberwithin{lemma}{subsection} 
\numberwithin{corollary}{subsection} 
\numberwithin{remark}{subsection} 
\numberwithin{equation}{subsection}	 

\subsection{Proof of Theorem~\ref{THMREGRET}}\label{app:Regret} 
 
In the following,  define $h_k(t)\triangleq \sqrt{\frac{(m+1)\log(t)}{T_k(t)}}$. 
The regret can be bounded  following the definition of $\Delta_{max}$ as
\begin{align} 
    R^{\cP}(n) &= \E\left[\sum_{t=1}^n r(\cG_t,t) - r(\cP_t,t)\right]  \notag \\ 
    &\leq \Delta_{\max}\E\left[\sum_{t=1}^n  \1 (r(\cG_t,t) \neq r(\cP_t,t)\right],
     \label{eq: opening inequality} 
\end{align}
We introduce another indicator function for every $i \in \bK$ along with its cumulative sum, denoted:
\begin{align*}
  I_i(t) &\triangleq \left.
  \begin{cases}
    1, & 
          \begin{cases} 
                & i = \underset{k \in \cC_t}{\argmin}~  T_k(t-1) \\
                & r(\cP_t,t) \neq r(\cG_t,t)  
          \end{cases} \\
    0, & \text{else}
  \end{cases}\right.
\,
N_i(n) \triangleq \sum_{t=1}^n I_i(t).
\end{align*}

Let $\cC_t \triangleq \cP_t \cup \cG_t$. 
In every round $t$ where $r(\cP_t) \neq r(\cG_t)$, the counter $N_k(t)$ is incremented for only a single user in $\cC_t$, while for the remaining users $N_k(t-1) = N_k(t)$. Thus, it holds that
$\sum_{t=1}^n  \1 \big ( (\cG_t,t) \neq r(\cP_t,t)\big ) = \sum_{k=1}^K N_k(n)$.
Substituting this  into~\eqref{eq: opening inequality}, we obtain that
\begin{equation}\label{eq: result of part 1}
    R^{\cP}(n) \leq \Delta_{\max}  \sum_{k=1}^K \E[N_k(n)].
\end{equation}

In the remainder, we focus on bounding $\E[N_k(n)]$ for every $k \in \bK$. After that, we substitute the derived upper bound into~\eqref{eq: result of part 1}.
To that aim, let $k \in \bK$ and fix some $l \in \mathbb{N}$ whose value is determined later. We note that:
\begin{align}
  &  \E[N_k(n)] = \E[\sum_{t=1}^n \1(I_k(t) = 1)]  \notag \\
    &=\E[\sum_{t=1}^n \1 (I_k(t) = 1, N_k(t) \leq l) + \1(I_k(t) = 1, N_k(t) > l)] \notag \\
    &\qquad\qquad  \stackrel{(a)}{\leq} l + \E[\sum_{t=1}^n \1(I_k(t) = 1, N_k(t) > l)],
\label{eq: bounding the counter}
\end{align}
where $(a)$ arises from  considering the cases $N_k(n) \leq l$ and its complementary state.

\ac{pause}'s policy \eqref{eq: pause's policy} implies that in every iteration:
\begin{align}
    &\!\!\min_{k \in \cP_t} {\rm ucb}(k,t-1) \! + \!\alpha \Phi_{\rm g}\left(\{g_k(t-1)\}_{k\in \cP_t},   \textcolor{NewColor2}{\cP_t}\right) 
    + \notag \\ 
    &\gamma \Phi_{\rm p}(\{p_k(t-1)\}_{k \in \cP_t}, \textcolor{NewColor2}{\cP_t})) \geq \notag \\ 
    &\!\!\min_{k \in \cG_t} {\rm ucb}(k,t\!-\!1) \!+ \!\alpha \Phi_{\rm g}\left(\{g_k(t-1)\}_{k\in \cG_t}, \textcolor{NewColor2}{\cG_t}\right) 
    +\notag \\ 
    & \gamma \cdot \Phi_{\rm p}(\{p_k(t-1)\}_{k \in \cG_t}, \textcolor{NewColor2}{\cG_t})).
    \label{eqn:Event_wp_1}
\end{align}
For the sake of readability we next abbreviate $\alpha\cdot \Phi_{\rm g}\left(\{g_k(t-1)\}_{k\in \cS}, \textcolor{NewColor2}{\cS}\right) 
    + \gamma \cdot \Phi_{\rm p}(\{p_k(t-1)\}_{k \in \cS}, \textcolor{NewColor2}{\cS})$ as $\sum_{\rm w \in p,g} \frac{\alpha_{\rm w}}{m} \Phi_{\rm w}(\{\rm w_k(t-1)\}_{k \in \cS})$, where $\alpha_p = \alpha$, and $\alpha_g = \gamma$.
Since the above-mentioned happens with probability one, we can incorporate it into the mentioned inequality \eqref{eq: bounding the counter} along with the featured notation:
\begin{equation*}
\begin{split}
    &\E[N_k(n)] \leq  l +\E\Biggl[\sum_{t=1}^n \1\Biggl\{I_k(t) = 1, N_k(t) > l, \\ &\min_{k \in \cP_t} {\rm ucb}(k,t-1) + \sum_{\rm w \in p,g} \frac{\alpha_{\rm w}}{m} \Phi_{\rm w}(\{\rm w_k(t-1)\}_{k \in \cP_t}) \geq \\ &\min_{k \in \cG_t} {\rm ucb}(k,t-1) + \sum_{\rm w \in p,g} \frac{\alpha_{\rm w}}{m} \Phi_{\rm w}(\{\rm w_k(t-1)\}_{k \in \cG_t})\Biggr\} \Biggr].
\end{split}
\end{equation*}

We now denote the users chosen in the $t$th iteration by the \ac{pause} algorithm and by the Genie as:
%
    $\cG_t = \Tilde{u}_{t,1},...,\Tilde{u}_{t,m}$ and
    $\cP_t = u_{t,1},...,u_{t,m}$, respectively.
%
%
For every $t$, the indicator function in the sum is equal to 1 only if the $k$th user is chosen the least at the beginning of the $t$th iteration, i.e., $T_k(t-1) \leq T_j(t-1)$ for every $j \in \cC_t$. The intersection of $I_k(t) = 1$ with $N_k(t) > l$ implies $T_k(t-1) \leq l$. Therefore, this intersection of events implies that for every $j \in \cC_t, l \leq T_j(t-1) \leq t-1$. 
Using this result, we can further bound every event in the indicator functions in the upper bound of $\E[N_k(n)]$:
\begin{equation*}
\begin{split}
     \E[N_k(n)] \leq 
     &l +\E\Biggl [\sum_{t=1}^n \1\biggl\{I_k(t) = 1, N_k(t) > l, \\ &\min_{l \leq T_{u_{t,1}},...,T_{u_{t,m}} \leq t-1} \biggr(\min_{k \in \cP_t} {\rm ucb}(k,t-1) + \\ &+ \sum_{\rm w \in p,g} \frac{\alpha_{\rm w}}{m} \Phi_{\rm w}(\{\rm w_k(t-1)\}_{k \in \cP_t}) \biggl)\geq  \\ & \min_{l \leq T_{\Tilde{u}_{t,1}},...,T_{\Tilde{u}_{t,m}} \leq t-1} \biggl( \min_{k \in \cG_t} {\rm ucb}(k,t-1) + \\ &\sum_{\rm w \in p,g} \frac{\alpha_{\rm w}}{m} \Phi_{\rm w}(\{\rm w_k(t-1)\}_{k \in \cG_t})\biggr)\biggr\} \biggr].
\end{split}
\end{equation*}

Using the fact that for any finite set of events of size $q$, it holds that $\{A_i\}_{i=1}^q$, $\1(\cup_{i=1}^q A_i) \leq \sum_{i=1}^q \1(A_i)$, and that the expectation of an indicator function is the probability of the internal event occurring, we have that
\begin{equation}\label{eq: split before upper bounding}
\begin{split}
    &\E[N_k(n)] \leq l + \sum_{t=1}^n\sum_{l \leq T_{\Tilde{u}_{t,1}},...,T_{\Tilde{u}_{t,m}}, T_{u_{t,1}},...,T_{u_{t,m}} \leq t-1} \\ &\bP\biggl[\min_{k \in \cP_t} {\rm ucb}(k,t\!-\!1) \!+\! \sum_{\rm w \in p,g} \frac{\alpha_{\rm w}}{m} \Phi_{\rm w}(\{\rm w_k(t-1)\}_{k \in \cP_t}) \geq \\ & \min_{k \in \cG_t} {\rm ucb}(k,t\!-\!1)\! +\! \sum_{\rm w \in p,g} \frac{\alpha_{\rm w}}{m} \Phi_{\rm w}(\{\rm w_k(t-1)\}_{k \in \cG_t})\biggr].
\end{split}
\end{equation}

In the following steps, we focus on bounding the terms in the double sum. To that aim,
we define the following:
\begin{equation}\label{eq: definitions of a_t and b_t}
    a_t \!=\! \argmin_{k\in\cP_t} {\rm ucb}(k,t\!-\!1), ~ b_t \!=\! \argmin_{k\in\cG_t} {\rm ucb}(k,t\!-\!1).
\end{equation}
Using these notations and writing $h_{a_t} \triangleq h_{a_t}(t)$, we state the following lemma:
\begin{lemma}\label{lm: contained events - 1}
The  event \eqref{eqn:Event_wp_1} implies that at least one of the next three events occurs:
\begin{enumerate}
    \item $\Bar{x}_{b_t} + h_{b_t} \leq \mu_{b_t}$;
    \item $\Bar{x}_{a_t} \geq \mu_{a_t} + h_{a_t}$;
    \item $
    \mu_{b_t} + \sum_{\rm w \in p,g} \frac{\alpha_{\rm w}}{m} \Phi_{\rm w}(\{\rm w_k(t-1)\}_{k \in \cG_t}) <  \mu_{a_t} + 2h{a_t} + \sum_{\rm w \in p,g} \frac{\alpha_{\rm w}}{m} \Phi_{\rm w}(\{\rm w_k(t-1)\}_{k \in \cP_t})$.
\end{enumerate}
\end{lemma}

\begin{proof}
Proof by contradiction: we assume all three events don't occur and examine the following:
\begin{equation*}
\begin{split}
    &\Bar{x}_{b_t} + h_{b_t} + \sum_{\rm w \in p,g} \frac{\alpha_{\rm w}}{m} \Phi_{\rm w}(\{\rm w_k(t-1)\}_{k \in \cG_t})  \stackrel{(1)}{>} \\ &\mu_{b_t} + \sum_{\rm w \in p,g} \frac{\alpha_{\rm w}}{m} \Phi_{\rm w}(\{\rm w_k(t-1)\}_{k \in \cG_t}) \stackrel{(3)}{\geq} \\ &\mu_{a_t} + 2h_{a_t} + \sum_{\rm w \in p,g} \frac{\alpha_{\rm w}}{m} \Phi_{\rm w}(\{\rm w_k(t-1)\}_{k \in \cP_t}) \stackrel{(2)}{>} \\ &\Bar{x}_{a_t} + h_{a_t} + \sum_{\rm w \in p,g} \frac{\alpha_{\rm w}}{m} \Phi_{\rm w}(\{\rm w_k(t-1)\}_{k \in \cP_t}).
\end{split}
\end{equation*}
By the definitions of $a_t$ and $b_t$~\eqref{eq: definitions of a_t and b_t}, the inequality above can also be written as:
\begin{align} 
    &\!\!\min_{k \in \cP_t} {\rm ucb}(k,t\!-\!1) \!+ \sum_{\rm w \in p,g} \frac{\alpha_{\rm w}}{m} \Phi_{\rm w}(\{\rm w_k(t-1)\}_{k \in \cP_t}) < \notag \\ &\!\!\min_{k \in \cG_t} {\rm ucb}(k,t\!-\!1)\! + \sum_{\rm w \in p,g} \frac{\alpha_{\rm w}}{m} \Phi_{\rm w}(\{\rm w_k(t-1)\}_{k \in \cG_t}),
\label{eq: the primary event}
\end{align}
contradicting our initially assumed event \eqref{eqn:Event_wp_1}.
\end{proof}

Applying the union bound and the relationship between the events shown in Lemma \ref{lm: contained events - 1} implies:
\begin{align}
    &\bP\biggl[\min_{k \in \cP_t} {\rm ucb}(k,t-1) + \sum_{\rm w \in p,g} \frac{\alpha_{\rm w}}{m} \Phi_{\rm w}(\{\rm w_k(t-1)\}_{k \in \cP_t}) \notag \\ &\geq \min_{k \in \cG_t} {\rm ucb}(k,t-1) + \sum_{\rm w \in p,g} \frac{\alpha_{\rm w}}{m} \Phi_{\rm w}(\{\rm w_k(t-1)\}_{k \in \cG_t})\biggr] \notag \\ &\leq \overbrace{\bP[\Bar{x}_{b_t} + h_{b_t} \leq \mu_{b_t}]}^{\triangleq (1)} + \overbrace{\bP[\Bar{x}_{a_t} \geq \mu_{a_t} + h_{a_t}]}^{\triangleq (2)} + \notag \\ &\bP[\mu_{b_t} + \sum_{\rm w \in p,g} \frac{\alpha_{\rm w}}{m} \Phi_{\rm w}(\{\rm w_k(t-1)\}_{k \in \cG_t}) < \notag \\ &\underbrace{\mu_{a_t}\! +\! 2h{a_t} \!+ \sum_{\rm w \in p,g} \frac{\alpha_{\rm w}}{m} \Phi_{\rm w}(\{\rm w_k(t-1)\}_{k \in \cP_t})]}_{\triangleq (3)}.
    \label{eq: splitting the event with union bound} 
\end{align}

We obtained three probability terms -- $(1)$, $(2)$, and $(3)$. We will start with bounding the first two using Hoeffding's inequality~\cite{hoeffding1994probability}. Term $(3)$ will be bounded right after in a different manner. We'll demonstrate how the first term is bounded; the second one is done similarly by replacing $b_t$ with $a_t$:
\begin{align} 
    &\bP[\Bar{x}_{b_t} + h_{b_t} \leq \mu_{b_t}] = \bP[\Bar{x}_{b_t} - \mu_{b_t} \leq -h_{b_t}] \notag\\ &=\bP\left[\sum_{j=1}^{T_{b_t}(t-1)} \frac{\tau_{min}}{(\tau_{b_t})_j} - \mu_{b_t} \leq -h_{b_t}T_{b_t}(t-1)\right]\notag   \\ &\leq e^-\frac{2T_{b_t}^2(t-1)(m+1)\log(t)}{T_{b_t}^2(t-1)} = t^{-2(m+1)},\label{eq: hoffeding's trick}
\end{align}
where $(\tau_{b_t})_j$ is the  latency of the user $b_t$ at the $j$th round it participated.
This results in the following inequalities:
\begin{equation*}
    \overbrace{\bP[\Bar{x}_{b_t} \!+\! h_{b_t} \leq \mu_{b_t}]}^{ = (1)} \leq  t^{-2(m\!+\!1)},~ \overbrace{\bP[\Bar{x}_{a_t} \geq \mu_{a_t} \!+ \!h_{a_t}]}^{ = (2)} \leq  t^{-2(m\!+\!1)}.
\end{equation*}

To bound $(3)$ we define another two definitions:
\begin{equation}\label{eq: definitions of A_t and B_t}
\begin{split}
    &A_t = \argmin_{k\in\cP_t} \mu_k, \quad B_t = \argmin_{k\in\cG_t} \mu_k.
\end{split}
\end{equation}
Using the law of total probability to divide $(3)$ into 2 parts:
\begin{equation*}
\begin{split}
    &\bP\biggr[\mu_{b_t} + \sum_{\rm w \in p,g} \frac{\alpha_{\rm w}}{m} \Phi_{\rm w}(\{\rm w_k(t-1)\}_{k \in \cG_t}) < \\ &\underbrace{\mu_{a_t} + 2h{a_t} + \sum_{\rm w \in p,g} \frac{\alpha_{\rm w}}{m} \Phi_{\rm w}(\{\rm w_k(t-1)\}_{k \in \cP_t})\biggr]}_{\triangleq (3)} \\ &= \bP\biggl[\bigl(\mu_{b_t} + \sum_{\rm w \in p,g} \frac{\alpha_{\rm w}}{m} \Phi_{\rm w}(\{\rm w_k(t-1)\}_{k \in \cG_t}) < \\ &\mu_{a_t} + 2h{a_t} + \sum_{\rm w \in p,g} \frac{\alpha_{\rm w}}{m} \Phi_{\rm w}(\{\rm w_k(t-1)\}_{k \in \cP_t})\bigr)  \\
    &\cap (b_t = B_t) \biggr] \\ &+ \bP\biggl[\bigl(\mu_{b_t} + \sum_{\rm w \in p,g} \frac{\alpha_{\rm w}}{m} \Phi_{\rm w}(\{\rm w_k(t-1)\}_{k \in \cG_t}) < \\ &\mu_{a_t} + 2h{a_t} + \sum_{\rm w \in p,g} \frac{\alpha_{\rm w}}{m} \Phi_{\rm w}(\{\rm w_k(t-1)\}_{k \in \cP_t})\bigr) \\
    &\cap (b_t \neq B_t) \biggr]. 
\end{split}
\end{equation*}
We denote the former term as $(3a)$ and the latter as $(3b)$:
\begin{subequations}
    \label{eq: definitions of 3a and 3b}
\begin{align}
    &(3a) \triangleq \bP\biggl[\bigl(\mu_{b_t} + \sum_{\rm w \in p,g} \frac{\alpha_{\rm w}}{m} \Phi_{\rm w}(\{\rm w_k(t-1)\}_{k \in \cG_t}) < \notag\\ 
    &\mu_{a_t} + 2h{a_t} + \sum_{\rm w \in p,g} \frac{\alpha_{\rm w}}{m} \Phi_{\rm w}(\{\rm w_k(t-1)\}_{k \in \cP_t})\bigr) \notag\\ 
    &\cap (b_t = B_t) \biggr],  \\
    &(3b) \triangleq \bP\biggl[\bigl(\mu_{b_t} + \sum_{\rm w \in p,g} \frac{\alpha_{\rm w}}{m} \Phi_{\rm w}(\{\rm w_k(t-1)\}_{k \in \cG_t}) < \notag\\\ 
    &\mu_{a_t} + 2h{a_t} + \sum_{\rm w \in p,g} \frac{\alpha_{\rm w}}{m} \Phi_{\rm w}(\{\rm w_k(t-1)\}_{k \in \cP_t})\bigr) \notag \\ 
    &\cap (b_t \neq B_t) \biggr].
\end{align}
\end{subequations}

In the following, we show that for a range of values of $l$, which so far was arbitrarily chosen, $3(a)$ is equal to $0$. Recalling the definitions of $a_t$~\eqref{eq: definitions of a_t and b_t} and $A_t$ \eqref{eq: definitions of A_t and B_t}, we know $\mu_{A_t} \leq \mu_{a_t}$. plugging this relation into probability of contained events in $(b)$, and upper bounding by omitting the intersection in $(a)$, yields:
\begin{equation*}
\begin{split}
    &(3a) \stackrel{(a)}{\leq} \\
    &\bP\biggl[\mu_{B_t} + \sum_{\rm w \in p,g} \frac{\alpha_{\rm w}}{m} \Phi_{\rm w}(\{\rm w_k(t-1)\}_{k \in \cG_t}) < 
    \\ &\mu_{a_t} + 2h{a_t} + \sum_{\rm w \in p,g} \frac{\alpha_{\rm w}}{m} \Phi_{\rm w}(\{\rm w_k(t-1)\}_{k \in \cP_t})\biggr] \\
    &\stackrel{(b)}{\leq} \bP\biggl[\mu_{B_t} + \sum_{\rm w \in p,g} \frac{\alpha_{\rm w}}{m} \Phi_{\rm w}(\{\rm w_k(t-1)\}_{k \in \cG_t}) < 
    \\ &\mu_{A_t} + 2h{a_t} + \sum_{\rm w \in p,g} \frac{\alpha_{\rm w}}{m} \Phi_{\rm w}(\{\rm w_k(t-1)\}_{k \in \cP_t})\biggr] \\
    &= \bP\biggl[\overbrace{\mu_{B_t} + \sum_{\rm w \in p,g} \frac{\alpha_{\rm w}}{m} \Phi_{\rm w}(\{\rm w_k(t-1)\}_{k \in \cG_t})}^{=C^{\cG}(\cG_t, t)} - 
    \\ &\overbrace{\bigl(\mu_{A_t} + \sum_{\rm w \in p,g} \frac{\alpha_{\rm w}}{m} \Phi_{\rm w}(\{\rm w_k(t-1)\}_{k \in \cP_t})\bigr)}^{=C^{\cG}(\cP_t, t)} < 2h_{a_t} \biggr] \\
    &= \bP\Bigl[C^{\cG}(\cG_t, t)-C^{\cG}(\cP_t, t) < 2\sqrt{\frac{(m+1)\log(t)}{T_{a_t}(t-1)}}\Bigr],
\end{split}
\end{equation*}
where the last two equalities derive from reorganizing the event and recalling the definitions of $C^{\cG}(\cS,t)$ and $h_k(t)$, respectively. We now show this event exists in probability 0, and then the latest bound implies $(3a)$ is equal to $0$ as well. We observe the mentioned event while recalling that $T_{a_t}(t) \geq l$ by the relevant indexes in the summation in \eqref{eq: split before upper bounding}:
\begin{align} 
    C^{\cG}(\cG_t, t)-C^{\cG}(\cP_t, t) &< 2\sqrt{\frac{(m+1)\log(t)}{T_{a_t}(t-1)}} \notag  \\
    &\leq2\sqrt{\frac{(m+1)\log(n)}{l}}.
    \label{eq: final event of 3a} 
\end{align}

Next, we observe an enhanced version of the Genie that is rewarded by an additive term of $\delta$ in every round that $\cG_t \neq \cP_t$. Recalling that we observe solely cases where this statement occurs, the LHS is directly larger than $\delta$. Thus, to secure the non-existence of this event, we may set any $l$ value fulfilling $\delta < 2\sqrt{\frac{(m+1)\log(n)}{l}}$. Recalling $\delta > 0$, we reorganize this condition into:
\begin{equation}\label{eq: first requirement on l}
    l \geq \Bigg\lceil \frac{4(m+1)\log(n)}{\delta^2} \Bigg\rceil.
\end{equation}

Moreover, this enhanced version adds another term of $\delta\sum_{t=1}^n  \E \bigg [\1 \Big ( (\cG_t,t) \neq r(\cP_t,t)\Big ) \bigg ] = \delta \sum_{k=1}^K \E[N_k(n)]$ to the regret, as noted later in the proof closure. 

Recall that we initially aimed to upper bound the probability of the event \eqref{eq: the primary event} by splitting it into three events using the union bound \eqref{eq: splitting the event with union bound}. We then showed $(1)$ and $(2)$ are bounded, and divided $(3)$ into 2 parts - $3(a)$ and $3(b)$. By setting an appropriate value of $l$ \eqref{eq: first requirement on l}, we demonstrated $3(a)$ can be shown to be equal to 0. The last step is to upper bound $3(b)$, which is done similarly.

We start by recalling the definition of $3(b)$ \eqref{eq: definitions of 3a and 3b} and then bound it by a containing event:
\begin{align} 
    &(3b) \triangleq \bP\biggl[\bigl(\mu_{b_t} + \sum_{\rm w \in p,g} \frac{\alpha_{\rm w}}{m} \Phi_{\rm w}(\{\rm w_k(t-1)\}_{k \in \cG_t}) < \notag\\ 
    &\qquad \mu_{a_t} + 2h{a_t} + \sum_{\rm w \in p,g} \frac{\alpha_{\rm w}}{m} \Phi_{\rm w}(\{\rm w_k(t-1)\}_{k \in \cP_t})\bigr) \notag\\ 
    &\qquad \cap (b_t \neq B_t) \biggr] \notag\\
    &\leq \bP[b_t \neq B_t] = \bP[\overline{\mu_{b_t}}(t) + h_{b_t} \leq \overline{\mu_{B_t}}(t) + h_{B_t}].\label{eq: first bounding 3b}
\end{align}

The last equality arises from the definitions of $b_t$ \eqref{eq: definitions of a_t and b_t} and $B_t$  \eqref{eq: definitions of A_t and B_t}, and definition \eqref{eq: ucb definition}. We now prove a lemma regarding this event, whose probability upper bounds $3(b)$:
\begin{lemma}
    The following event implies that at least one of the next three events occurs:
    \begin{equation}
    \overline{\mu_{b_t}}(t) + h_{b_t} \leq \overline{\mu_{B_t}}(t) + h_{B_t}  
    \end{equation}
    \begin{enumerate}
        \item $\overline{\mu_{b_t}}(t) + h_{b_t} \leq \mu_{b_t}$
        \item $\overline{\mu_{B_t}}(t) \geq \mu_{B_t} + h_{B_t}$
        \item $\mu_{b_t} < \mu_{B_t} + 2h_{B_t}$
    \end{enumerate}
\end{lemma}

\begin{proof}
    We  prove by contradiction, as $\overline{\mu_{b_t}}(t) + h_{b_t} \stackrel{(1)}{>} \mu_{b_t} \stackrel{(3)}{\geq} \mu_{B_t} + 2h_{B_t} \stackrel{(2)}{>} \overline{\mu_{B_t}}(t) + h_{B_t}$, 
thus proving the lemma
\end{proof}

Combining the lemma, the union bound, and the upper bound we found in \eqref{eq: first bounding 3b} yields:
\begin{align*}
    3(b) \leq &\bP[\overline{\mu_{b_t}}(t) - \mu_{b_t} \leq -h_{b_t}]  + \bP[\overline{\mu_{B_t}}(t) - \mu_{B_t} \geq h_{B_t}] \\
    &+ \bP[\mu_{b_t} < \mu_{B_t} + 2h_{B_t}].
\end{align*}

We already showed in \eqref{eq: hoffeding's trick} that the first term is bounded by $t^{-2(m+1)}$. Repeating the same steps for $B_t$ instead of $b_t$ we can show that this value also bounds the second term. Furthermore, we now show that the event in the third term occurs with probability 0 when setting an appropriate value of $l$. Observing the mentioned event:
\begin{equation}
\begin{split}
    \mu_{b_t} - \mu_{B_t} < 2\frac{(m+1)log(t)}{T_{B_t}(t-1)}.
\end{split}
\end{equation}

Similar to \eqref{eq: first requirement on l}, and recalling $\delta$'s definition and $b_t \neq B_t$, by demanding $l \geq \Big\lceil \frac{4(m+1)\log(n)}{\delta^2} \Big\rceil$ we assure this event occurs with probability 0. As this is the same range as in \eqref{eq: first requirement on l}, we set $l$ to be the lowest integer in this range, i.e., $l = \Big\lceil \frac{4(m+1)\log(n)}{\delta^2} \Big\rceil$.

Finally, as we showed: $(3) \leq 2t^{-2(m+1)}$. plugging the bounds on $(1)$, $(2)$, and $(3)$ into \eqref{eq: splitting the event with union bound} we obtain:
\begin{equation*}
\begin{split}
    &\bP\biggl[\min_{k \in \cP_t} {\rm ucb}(k,t-1) + \sum_{\rm w \in p,g} \frac{\alpha_{\rm w}}{m} \Phi_{\rm w}(\{\rm w_k(t-1)\}_{k \in \cP_t}) \\ 
    &\geq \min_{k \in \cG_t} {\rm ucb}(k,t-1) + \sum_{\rm w \in p,g} \frac{\alpha_{\rm w}}{m} \Phi_{\rm w}(\{\rm w_k(t-1)\}_{k \in \cG_t})\biggr] \\
    &\leq \overbrace{t^{-2(m+1)}}^{\geq (1)} + \overbrace{t^{-2(m+1)}}^{\geq (2)} + 
    \overbrace{2t^{-2(m+1)}}^{\geq (3)} = 4t^{-2(m+1)}.
\end{split}
\end{equation*}
Substituting this bound along with the chosen value of $l$ into the result we obtained at the beginning of the proof \eqref{eq: split before upper bounding}, we obtain:
\begin{equation*}
\begin{split}
    \E[N_k(n)] \leq  
    &\Big\lceil \frac{4(m+1)\log(n)}{\delta^2} \Big\rceil + \\
    &\sum_{t=1}^n\sum_{l \leq T_{\Tilde{u}_{t,1}},...,T_{\Tilde{u}_{t,m}}, T_{u_{t,1}},...,T_{u_{t,m}} \leq t-1} 4t^{-2(m+1)} \\
    &\leq \frac{4(m+1)\log(n)}{\delta^2} + 1 + \sum_{t=1}^n 4t^{-2(m+1)} \cdot t^{2m} \\ 
    &\leq \frac{4(m+1)\log(n)}{\delta^2} + 1 + 4\overbrace{\sum_{t=1}^{\infty} t^{-2}}^{=\pi^2/3}.
\end{split}
\end{equation*}

To conclude the theorem's statement, we set this result back into \eqref{eq: result of part 1} while recalling the added regret from the Genie empowerment, obtaining
\begin{equation*}
\begin{split}
    R^{\cP}(n) &\leq (\Delta_{\max} + \delta)  \sum_{k=1}^K \E[N_k(n)]\\
    & \leq K (\Delta_{\max} + \delta) \bigg(\frac{4(m+1)\log(n)}{\delta^2} + 1 + \frac{4\pi^2}{3}\bigg),
\end{split}    
\end{equation*}
concluding the proof of the theorem.

\subsection{Proof of Theorem \ref{THMSA}}\label{app:sa convergence} 
To prove the theorem, we introduce essential terminology and definitions. We define reachability as follows: Given two nodes $\mySet{V}_1$ and $\mySet{V}_2$ and energy level $E$, node $\mySet{V}_1$ is considered reachable from $\mySet{V}_2$ if there exists a path connecting them that traverses only nodes with energy greater than or equal to $E$. Building upon this definition, a graph exhibits Weak Reversibility if, for any energy level $E$ and nodes $\mySet{U}_1$ and $\mySet{U}_2$, $\mySet{U}_1$ is reachable from $\mySet{U}_2$ at height $E$ if and only if $\mySet{U}_2$ is reachable from $\mySet{U}_1$ at height $E$.

Following \cite{hajek1988cooling}, to prove that Theorem \ref{THMSA} holds, one has to show that the following requirements hold: 
     \begin{enumerate}[label={R\arabic*}]
         \item \label{itm:Req1} The graph satisfies weak reversibility \cite{hajek1988cooling}.
         \item \label{itm:Req2} The temperature sequence is from the form of $\tau_j = \frac{C}{\log(j+1)}$ where $C$ is greater than the maximal energy difference between any two nodes.
         \item \label{itm:Req3}  The Markov chain introduced in Algorithm \ref{alg:tailored-SA} is irreducible.
     \end{enumerate}
     
We prove the three mentioned conditions are satisfied to conclude the theorem. Requirements \ref{itm:Req1} and \ref{itm:Req2} follow from the formulation of SA-PAUSE. Specifically,  weak reversibility (\ref{itm:Req1}) stems directly from the definition and the undirected graph property, while the temperature sequence condition \ref{itm:Req2} is satisfied as we set $C$ to be as mentioned in~\eqref{eq: setting C}.

    To prove that \ref{itm:Req3} holds, by definition, we need to show there is a path with positive probability between any two nodes $\mySet{V},\mySet{U} \in \bV$. Since the graph is undirected, it is sufficient to show a path from $\mySet{V}$ to $\mySet{U}$. In Algorithm~\ref{alg: second appendix proof}, we present an implicit algorithm yielding a series of nodes $\mySet{V}_0, \mySet{V}_1, \ldots, \mySet{U}$. within this sequence, consecutive nodes are neighbors, i.e., the algorithm yields a path with positive probability from $\mySet{V}_0$ to $\mySet{U}$. 

    \begin{algorithm}
        \caption{Constructing Path  from $\mySet{V}_0$ to $\mySet{U}$ }
    \label{alg: second appendix proof}
        \SetKwInOut{Input}{Input} 
        \Input{Set of users $\bK$; an arbitrary node $\mySet{V}_0$, and $\mySet{U}$}
        
        \SetKwInOut{Initialization}{Init}
        \Initialization{$j = 0$}
        {

            \While{$\mySet{U} \neq \mySet{V}_j$}{ 
            \If{$\min_{k \in \mySet{V}_j} \{{\rm ucb}(k)\} \leq \max_{k \in \mySet{U} \setminus \mySet{V}_j} \{{\rm ucb}(k)\} $}{$\mySet{V}_{j+1} \triangleq \Big(\mySet{V}_j \setminus \argmin_{k \in \mySet{V}_j} \{{\rm ucb}(k)\}\Big) \bigcup \argmax_{k \in \mySet{U} \setminus \mySet{V}_j} \{{\rm ucb}(k)\} $}

            \Else{sample a random user $p$ from $\mySet{V}_j \setminus \mySet{U}$;
            \newline $\mySet{V}_{j+1} \triangleq (\mySet{V}_j \setminus \{p\}) \bigcup \argmax_{k \in \mySet{U} \setminus \mySet{V}_j} \{{\rm ucb}(k)\} $;
            \newline $j = j+1$

            }

            }
      }
    \end{algorithm}

    This algorithm possesses a crucial characteristic; the conditional statement evaluates to true until it transitions to false, and from that moment on, it remains False to the end. Thus, the algorithm can be partitioned into two phases: the iterations before the statement becomes false, and the rest. We denote the iteration the condition becomes false as $j^0$.

    First, observe that when $j<j^0$, $\mySet{V}_{j+1}$ is an active neighbor of $\mySet{V}_{j}$, whereas during all subsequent iterations, the former is a passive neighbor of the latter. This proves the transitions occur with positive probability in the first place. 

    Next, we prove the algorithm's correctness and termination. Let $b$ the minimum ${\rm ucb}$ value in $\mySet{V}_{j^0}$. For every $k \in \mySet{U}$, if ${\rm ucb}(k) > b$, then it is added to $\mySet{V}_j$ in an iteration $j<j^0$. this is guaranteed because if such incorporation had not occurred by the $j^0$th iteration, the conditional statement would remain satisfied, contradicting the definition of $b$. The rest of the users, i.e., every $k \in \mySet{U}$ such that ${\rm ucb}(k) \leq b$, will be added during the second phase. 

    Notice the algorithm avoids cyclical additions and subtractions, as during the second phase, users from $\mySet{U}$ who are already present in $\mySet{V}_j$ for all $j \geq j^0$ are preserved when constructing $\mySet{V}_{j+1}$. Instead, a user not belonging to $\mySet{U}$ is eliminated. Throughout this exposition, we have established that the algorithm terminates, and every user $k \in \mySet{U}$ is eventually incorporated into the evolving set without subsequent elimination. This completes our verification of the algorithm's correctness and the proof as a whole.

\end{appendix}

\fi

\bibliographystyle{IEEEtran}
\bibliography{IEEEabrv,IEEEabrv2,mybib}

\begin{thebibliography}{10}
\providecommand{\url}[1]{#1}
\csname url@samestyle\endcsname
\providecommand{\newblock}{\relax}
\providecommand{\bibinfo}[2]{#2}
\providecommand{\BIBentrySTDinterwordspacing}{\spaceskip=0pt\relax}
\providecommand{\BIBentryALTinterwordstretchfactor}{4}
\providecommand{\BIBentryALTinterwordspacing}{\spaceskip=\fontdimen2\font plus
\BIBentryALTinterwordstretchfactor\fontdimen3\font minus \fontdimen4\font\relax}
\providecommand{\BIBforeignlanguage}[2]{{%
\expandafter\ifx\csname l@#1\endcsname\relax
\typeout{** WARNING: IEEEtran.bst: No hyphenation pattern has been}%
\typeout{** loaded for the language `#1'. Using the pattern for}%
\typeout{** the default language instead.}%
\else
\language=\csname l@#1\endcsname
\fi
#2}}
\providecommand{\BIBdecl}{\relax}
\BIBdecl

\bibitem{peleg2024pause}
O.~Peleg, N.~Lang, S.~Rini, N.~Shlezinger, and K.~Cohen, ``{PAUSE}: Privacy-aware active user selection for federated learning,'' in \emph{IEEE International Conference on Acoustics, Speech and Signal Processing (ICASSP)}, 2025.

\bibitem{mcmahan2017communication}
B.~McMahan, E.~Moore, D.~Ramage, S.~Hampson, and B.~A. y~Arcas, ``Communication-efficient learning of deep networks from decentralized data,'' in \emph{Artificial Intelligence and Statistics}.\hskip 1em plus 0.5em minus 0.4em\relax PMLR, 2017, pp. 1273--1282.

\bibitem{kairouz2021advances}
P.~Kairouz \emph{et~al.}, ``Advances and open problems in federated learning,'' \emph{Foundations and trends{\textregistered} in machine learning}, vol.~14, no. 1--2, pp. 1--210, 2021.

\bibitem{gafni2022federated}
T.~Gafni, N.~Shlezinger, K.~Cohen, Y.~C. Eldar, and H.~V. Poor, ``Federated learning: A signal processing perspective,'' \emph{{IEEE} Signal Process. Mag.}, vol.~39, no.~3, pp. 14--41, 2022.

\bibitem{li2020federated}
T.~Li, A.~K. Sahu, A.~Talwalkar, and V.~Smith, ``Federated learning: Challenges, methods, and future directions,'' \emph{{IEEE} Signal Process. Mag.}, vol.~37, no.~3, pp. 50--60, 2020.

\bibitem{chen2021communication}
M.~Chen, N.~Shlezinger, H.~V. Poor, Y.~C. Eldar, and S.~Cui, ``Communication-efficient federated learning,'' \emph{Proceedings of the National Academy of Sciences}, vol. 118, no.~17, 2021.

\bibitem{alistarh2018convergence}
D.~Alistarh, T.~Hoefler, M.~Johansson, N.~Konstantinov, S.~Khirirat, and C.~Renggli, ``The convergence of sparsified gradient methods,'' \emph{Advances in Neural Information Processing Systems}, vol.~31, 2018.

\bibitem{lang2025olala}
N.~Lang, M.~Simhi, and N.~Shlezinger, ``{OLALa}: Online learned adaptive lattice codes for heterogeneous federated learning,'' \emph{arXiv preprint arXiv:2506.20297}, 2025.

\bibitem{han2020adaptive}
P.~Han, S.~Wang, and K.~K. Leung, ``Adaptive gradient sparsification for efficient federated learning: An online learning approach,'' in \emph{IEEE International Conference on Distributed Computing Systems (ICDCS)}, 2020, pp. 300--310.

\bibitem{reisizadeh2020fedpaq}
A.~Reisizadeh, A.~Mokhtari, H.~Hassani, A.~Jadbabaie, and R.~Pedarsani, ``Fedpaq: A communication-efficient federated learning method with periodic averaging and quantization,'' in \emph{International Conference on Artificial Intelligence and Statistics}.\hskip 1em plus 0.5em minus 0.4em\relax PMLR, 2020, pp. 2021--2031.

\bibitem{shlezinger2020uveqfed}
N.~Shlezinger, M.~Chen, Y.~C. Eldar, H.~V. Poor, and S.~Cui, ``{UVeQFed}: Universal vector quantization for federated learning,'' \emph{{IEEE} Trans. Signal Process.}, vol.~69, pp. 500--514, 2020.

\bibitem{amiri2020machine}
M.~M. Amiri and D.~G{\"u}nd{\"u}z, ``Machine learning at the wireless edge: Distributed stochastic gradient descent over-the-air,'' \emph{{IEEE} Trans. Signal Process.}, vol.~68, pp. 2155--2169, 2020.

\bibitem{sery2020analog}
T.~Sery and K.~Cohen, ``On analog gradient descent learning over multiple access fading channels,'' \emph{{IEEE} Trans. Signal Process.}, vol.~68, pp. 2897--2911, 2020.

\bibitem{yang2020federated}
K.~Yang, T.~Jiang, Y.~Shi, and Z.~Ding, ``Federated learning via over-the-air computation,'' \emph{{IEEE} Trans. Wireless Commun.}, vol.~19, no.~3, pp. 2022--2035, 2020.

\bibitem{mayhoub2024review}
S.~Mayhoub and T.~M.~Shami, ``A review of client selection methods in federated learning,'' \emph{Archives of Computational Methods in Engineering}, vol.~31, no.~2, pp. 1129--1152, 2024.

\bibitem{li2024comprehensive}
J.~Li, T.~Chen, and S.~Teng, ``A comprehensive survey on client selection strategies in federated learning,'' \emph{Computer Networks}, p. 110663, 2024.

\bibitem{10197174}
L.~Fu, H.~Zhang, G.~Gao, M.~Zhang, and X.~Liu, ``Client selection in federated learning: Principles, challenges, and opportunities,'' \emph{{IEEE} Internet Things J.}, vol.~10, no.~24, pp. 21\,811--21\,819, 2023.

\bibitem{xu2020client}
J.~Xu and H.~Wang, ``Client selection and bandwidth allocation in wireless federated learning networks: A long-term perspective,'' \emph{{IEEE} Trans. Wireless Commun.}, vol.~20, no.~2, pp. 1188--1200, 2020.

\bibitem{abdulrahman2020fedmccs}
S.~AbdulRahman, H.~Tout, A.~Mourad, and C.~Talhi, ``Fed{MCCS}: Multicriteria client selection model for optimal iot federated learning,'' \emph{{IEEE} Internet Things J.}, vol.~8, no.~6, pp. 4723--4735, 2020.

\bibitem{rizk2022federated}
E.~Rizk, S.~Vlaski, and A.~H. Sayed, ``Federated learning under importance sampling,'' \emph{{IEEE} Trans. Signal Process.}, vol.~70, pp. 5381--5396, 2022.

\bibitem{fraboni2021clustered}
Y.~Fraboni, R.~Vidal, L.~Kameni, and M.~Lorenzi, ``Clustered sampling: Low-variance and improved representativity for clients selection in federated learning,'' in \emph{International Conference on Machine Learning}.\hskip 1em plus 0.5em minus 0.4em\relax PMLR, 2021, pp. 3407--3416.

\bibitem{xia2020multi}
W.~Xia, T.~Q. Quek, K.~Guo, W.~Wen, H.~H. Yang, and H.~Zhu, ``Multi-armed bandit-based client scheduling for federated learning,'' \emph{{IEEE} Trans. Wireless Commun.}, vol.~19, no.~11, pp. 7108--7123, 2020.

\bibitem{xu2021online}
B.~Xu, W.~Xia, J.~Zhang, T.~Q. Quek, and H.~Zhu, ``Online client scheduling for fast federated learning,'' \emph{{IEEE} Wireless Commun. Lett.}, vol.~10, no.~7, pp. 1434--1438, 2021.

\bibitem{ami2023client}
D.~Ben-Ami, K.~Cohen, and Q.~Zhao, ``Client selection for generalization in accelerated federated learning: A multi-armed bandit approach,'' \emph{{IEEE} Access}, 2025.

\bibitem{chen2024personalized}
Y.~Chen, W.~Xu, X.~Wu, M.~Zhang, and B.~Luo, ``Personalized local differentially private federated learning with adaptive client sampling,'' in \emph{IEEE International Conference on Acoustics, Speech and Signal Processing (ICASSP)}, 2024, pp. 6600--6604.

\bibitem{huang2020efficiency}
T.~Huang, W.~Lin, W.~Wu, L.~He, K.~Li, and A.~Y. Zomaya, ``An efficiency-boosting client selection scheme for federated learning with fairness guarantee,'' \emph{{IEEE} Trans. Parallel Distrib. Syst.}, vol.~32, no.~7, pp. 1552--1564, 2020.

\bibitem{yang2021federated}
M.~Yang, X.~Wang, H.~Zhu, H.~Wang, and H.~Qian, ``Federated learning with class imbalance reduction,'' in \emph{European Signal Processing Conference (EUSIPCO)}.\hskip 1em plus 0.5em minus 0.4em\relax IEEE, 2021, pp. 2174--2178.

\bibitem{huang2022stochastic}
T.~Huang, W.~Lin, L.~Shen, K.~Li, and A.~Y. Zomaya, ``Stochastic client selection for federated learning with volatile clients,'' \emph{{IEEE} Internet Things J.}, vol.~9, no.~20, pp. 20\,055--20\,070, 2022.

\bibitem{shi2022vfedcs}
F.~Shi, C.~Hu, W.~Lin, L.~Fan, T.~Huang, and W.~Wu, ``{VFedCS}: Optimizing client selection for volatile federated learning,'' \emph{{IEEE} Internet Things J.}, vol.~9, no.~24, pp. 24\,995--25\,010, 2022.

\bibitem{wang2024fedmaba}
Z.~Wang, L.~Wang, Y.~Guo, Y.-J.~A. Zhang, and X.~Tang, ``Fed{MABA}: Towards fair federated learning through multi-armed bandits allocation,'' \emph{arXiv preprint arXiv:2410.20141}, 2024.

\bibitem{guo2024auction}
J.~Guo, L.~Su, J.~Liu, J.~Ding, X.~Liu, B.~Huang, and L.~Li, ``Auction-based client selection for online federated learning,'' \emph{Information Fusion}, vol. 112, p. 102549, 2024.

\bibitem{ZHU2024110512}
K.~Zhu, F.~Zhang, L.~Jiao, B.~Xue, and L.~Zhang, ``Client selection for federated learning using combinatorial multi-armed bandit under long-term energy constraint,'' \emph{Computer Networks}, vol. 250, p. 110512, 2024.

\bibitem{zhu2020deep}
L.~Zhu and S.~Han, ``Deep leakage from gradients,'' in \emph{Federated learning}.\hskip 1em plus 0.5em minus 0.4em\relax Springer, 2020, pp. 17--31.

\bibitem{zhao2020idlg}
B.~Zhao, K.~R. Mopuri, and H.~Bilen, ``{iDLG}: Improved deep leakage from gradients,'' \emph{arXiv preprint arXiv:2001.02610}, 2020.

\bibitem{huang2021evaluating}
Y.~Huang, S.~Gupta, Z.~Song, K.~Li, and S.~Arora, ``Evaluating gradient inversion attacks and defenses in federated learning,'' \emph{Advances in Neural Information Processing Systems}, vol.~34, 2021.

\bibitem{yin2021see}
H.~Yin, A.~Mallya, A.~Vahdat, J.~M. Alvarez, J.~Kautz, and P.~Molchanov, ``See through gradients: Image batch recovery via gradinversion,'' in \emph{Proceedings of the IEEE/CVF Conference on Computer Vision and Pattern Recognition}, 2021, pp. 16\,337--16\,346.

\bibitem{kim2021federated}
M.~Kim, O.~G{\"u}nl{\"u}, and R.~F. Schaefer, ``Federated learning with local differential privacy: Trade-offs between privacy, utility, and communication,'' in \emph{IEEE International Conference on Acoustics, Speech and Signal Processing (ICASSP)}, 2021, pp. 2650--2654.

\bibitem{wei2020federated}
K.~Wei \emph{et~al.}, ``Federated learning with differential privacy: Algorithms and performance analysis,'' \emph{{IEEE} Trans. Inf. Forensics Security}, vol.~15, pp. 3454--3469, 2020.

\bibitem{lyu2021dp}
L.~Lyu, ``{DP-SIGNSGD}: When efficiency meets privacy and robustness,'' in \emph{IEEE International Conference on Acoustics, Speech and Signal Processing (ICASSP)}, 2021, pp. 3070--3074.

\bibitem{lowy2021private}
A.~Lowy and M.~Razaviyayn, ``Private federated learning without a trusted server: Optimal algorithms for convex losses,'' in \emph{International Conference on Learning Representations}, 2023.

\bibitem{lang2022joint}
N.~Lang, E.~Sofer, T.~Shaked, and N.~Shlezinger, ``Joint privacy enhancement and quantization in federated learning,'' \emph{{IEEE} Trans. Signal Process.}, vol.~71, pp. 295--310, 2023.

\bibitem{lang2023compressed}
N.~Lang, N.~Shlezinger, R.~G. D'Oliveira, and S.~E. Rouayheb, ``Compressed private aggregation for scalable and robust federated learning over massive networks,'' \emph{{IEEE} Trans. Mobile Comput.}, 2025, early access.

\bibitem{dwork2010boosting}
C.~Dwork, G.~N. Rothblum, and S.~Vadhan, ``Boosting and differential privacy,'' in \emph{IEEE Annual Symposium on Foundations of Computer Science}, 2010, pp. 51--60.

\bibitem{zhang2024dynamic}
J.~Zhang, D.~Fay, and M.~Johansson, ``Dynamic privacy allocation for locally differentially private federated learning with composite objectives,'' in \emph{IEEE International Conference on Acoustics, Speech and Signal Processing (ICASSP)}, 2024, pp. 9461--9465.

\bibitem{sun2020ldp}
L.~Sun, J.~Qian, X.~Chen, and P.~S. Yu, ``{LDP-FL}: Practical private aggregation in federated learning with local differential privacy,'' in \emph{International Joint Conference on Artificial Intelligence}, 2021.

\bibitem{cheu2019distributed}
A.~Cheu, A.~Smith, J.~Ullman, D.~Zeber, and M.~Zhilyaev, ``Distributed differential privacy via shuffling,'' in \emph{Advances in Cryptology--EUROCRYPT 2019: 38th Annual International Conference on the Theory and Applications of Cryptographic Techniques, Darmstadt, Germany, May 19--23, 2019, Proceedings, Part I 38}.\hskip 1em plus 0.5em minus 0.4em\relax Springer, 2019, pp. 375--403.

\bibitem{balle2019privacy}
B.~Balle, J.~Bell, A.~Gasc{\'o}n, and K.~Nissim, ``The privacy blanket of the shuffle model,'' in \emph{Advances in Cryptology--CRYPTO 2019: 39th Annual International Cryptology Conference, Santa Barbara, CA, USA, August 18--22, 2019, Proceedings, Part II 39}.\hskip 1em plus 0.5em minus 0.4em\relax Springer, 2019, pp. 638--667.

\bibitem{zhao2022multi}
Q.~Zhao, \emph{Multi-armed bandits: Theory and applications to online learning in networks}.\hskip 1em plus 0.5em minus 0.4em\relax Springer Nature, 2022.

\bibitem{auer2002finite}
P.~Auer, N.~Cesa-Bianchi, and P.~Fischer, ``Finite-time analysis of the multiarmed bandit problem,'' \emph{Machine learning}, vol.~47, pp. 235--256, 2002.

\bibitem{chen2013combinatorial}
W.~Chen, Y.~Wang, and Y.~Yuan, ``Combinatorial multi-armed bandit: General framework and applications,'' in \emph{International conference on machine learning}.\hskip 1em plus 0.5em minus 0.4em\relax PMLR, 2013, pp. 151--159.

\bibitem{hajek1988cooling}
B.~Hajek, ``Cooling schedules for optimal annealing,'' \emph{Mathematics of operations research}, vol.~13, no.~2, pp. 311--329, 1988.

\bibitem{li2019convergence}
X.~Li, K.~Huang, W.~Yang, S.~Wang, and Z.~Zhang, ``On the convergence of {F}ed{A}vg on non-iid data,'' in \emph{International Conference on Learning Representations}, 2019.

\bibitem{lang2024stragglers}
N.~Lang, A.~Cohen, and N.~Shlezinger, ``Stragglers-aware low-latency synchronous federated learning via layer-wise model updates,'' \emph{IEEE Trans. on~Commun.}, vol.~73, no.~5, pp. 3333--3346, 2025.

\bibitem{kasiviswanathan2011can}
S.~P. Kasiviswanathan, H.~K. Lee, K.~Nissim, S.~Raskhodnikova, and A.~Smith, ``What can we learn privately?'' \emph{SIAM Journal on Computing}, vol.~40, no.~3, pp. 793--826, 2011.

\bibitem{wang2020federated}
Y.~Wang, Y.~Tong, and D.~Shi, ``Federated latent dirichlet allocation: A local differential privacy based framework,'' in \emph{Proceedings of the AAAI Conference on Artificial Intelligence}, vol.~34, no.~04, 2020, pp. 6283--6290.

\bibitem{wang2020comprehensive}
T.~Wang, X.~Zhang, J.~Feng, and X.~Yang, ``A comprehensive survey on local differential privacy toward data statistics and analysis,'' \emph{Sensors}, vol.~20, no.~24, p. 7030, 2020.

\bibitem{dwork2016calibrating}
C.~Dwork, F.~McSherry, K.~Nissim, and A.~Smith, ``Calibrating noise to sensitivity in private data analysis,'' \emph{Journal of Privacy and Confidentiality}, vol.~7, no.~3, pp. 17--51, 2016.

\bibitem{bonawitz2019towards}
K.~Bonawitz, H.~Eichner, W.~Grieskamp, D.~Huba, A.~Ingerman, V.~Ivanov, C.~Kiddon, J.~Kone{\v{c}}n{\`y}, S.~Mazzocchi, B.~McMahan \emph{et~al.}, ``Towards federated learning at scale: System design,'' \emph{Machine Learning and Systems (MLSys)}, vol.~1, pp. 374--388, 2019.

\bibitem{xie2019asynchronous}
C.~Xie, S.~Koyejo, and I.~Gupta, ``Asynchronous federated optimization,'' \emph{arXiv preprint arXiv:1903.03934}, 2019.

\bibitem{ortega2023asynchronous}
T.~Ortega and H.~Jafarkhani, ``Asynchronous federated learning with bidirectional quantized communications and buffered aggregation,'' in \emph{International Conference on Machine Learning (ICML), Workshop on Federated Learning and Analytics}, 2023.

\bibitem{henderson2003theory}
D.~Henderson, S.~H. Jacobson, and A.~W. Johnson, ``The theory and practice of simulated annealing,'' \emph{Handbook of metaheuristics}, pp. 287--319, 2003.

\bibitem{ledesma2008practical}
S.~Ledesma, G.~Avi{\~n}a, and R.~Sanchez, ``Practical considerations for simulated annealing implementation,'' \emph{Simulated annealing}, vol.~20, pp. 401--420, 2008.

\bibitem{ben2004computing}
W.~Ben-Ameur, ``Computing the initial temperature of simulated annealing,'' \emph{Computational optimization and applications}, vol.~29, pp. 369--385, 2004.

\bibitem{bezakova2008accelerating}
I.~Bez{\'a}kov{\'a}, D.~{\v{S}}tefankovi{\v{c}}, V.~V. Vazirani, and E.~Vigoda, ``Accelerating simulated annealing for the permanent and combinatorial counting problems,'' \emph{SIAM Journal on Computing}, vol.~37, no.~5, pp. 1429--1454, 2008.

\bibitem{wu2018adaptive}
H.~Wu, X.~Guo, and X.~Liu, ``Adaptive exploration-exploitation tradeoff for opportunistic bandits,'' in \emph{International Conference on Machine Learning}.\hskip 1em plus 0.5em minus 0.4em\relax PMLR, 2018, pp. 5306--5314.

\bibitem{drugan2014pareto}
M.~M. Drugan, A.~Now{\'e}, and B.~Manderick, ``Pareto upper confidence bounds algorithms: an empirical study,'' in \emph{IEEE Symposium on Adaptive Dynamic Programming and Reinforcement Learning (ADPRL)}, 2014.

\bibitem{hoeffding1994probability}
W.~Hoeffding, ``Probability inequalities for sums of bounded random variables,'' \emph{The collected works of Wassily Hoeffding}, pp. 409--426, 1994.

\end{thebibliography}

\end{document}